\documentclass{article}

\usepackage{microtype}
\usepackage{graphicx}
\usepackage{subfigure}
\usepackage{booktabs} 

\usepackage{hyperref}

\usepackage[utf8]{inputenc} 
\usepackage[T1]{fontenc}    
\usepackage{url}            
\usepackage{amsfonts}       
\usepackage{nicefrac}       
\usepackage{subfigure}
\graphicspath{{figures/}}
\usepackage{xcolor}
\usepackage{adjustbox}
\usepackage{bm}
\usepackage{amsmath}
\usepackage{amssymb}
\usepackage{algorithm}
\usepackage{algorithmic}

\input{Definitions}

\ifx\proof\undefined
\newenvironment{proof}{\par\noindent{\bf Proof\ }}{\hfill\BlackBox\\[2mm]}
\fi

\ifx\theorem\undefined
\newtheorem{theorem}{Theorem}
\fi
\ifx\example\undefined
\newtheorem{example}{Example}
\fi
\ifx\lemma\undefined
\newtheorem{lemma}[theorem]{Lemma}
\fi

\ifx\corollary\undefined
\newtheorem{corollary}[theorem]{Corollary}
\fi

\ifx\assumption\undefined
\newtheorem{assumption}{Assumption}
\fi

\ifx\definition\undefined
\newtheorem{definition}{Definition}
\fi

\ifx\proposition\undefined
\newtheorem{proposition}[theorem]{Proposition}
\fi

\ifx\remark\undefined
\newtheorem{remark}{Remark}
\fi

\ifx\conjecture\undefined
\newtheorem{conjecture}[theorem]{Conjecture}
\fi

\ifx\factoid\undefined
\newtheorem{factoid}[theorem]{Fact}
\fi

\ifx\axiom\undefined
\newtheorem{axiom}[theorem]{Axiom}
\fi

\newcommand{\thetav}{{\boldsymbol \theta}}

\newcommand{\RN}[1]{%
	\textup{\lowercase\expandafter{\it \romannumeral#1}}%
}

\def\KL{\textsf{KL}} 
\def\InvGamma{\textsf{Inv-Gamma}} 

\usepackage[accepted]{icml2018}

\icmltitlerunning{Policy Optimization as Wasserstein Gradient Flows}

\begin{document}
	
	\twocolumn[
	\icmltitle{Policy Optimization as Wasserstein Gradient Flows}
	
	
	
	
	\begin{icmlauthorlist}
		\icmlauthor{Ruiyi Zhang}{duke}
		\icmlauthor{Changyou Chen}{bu}
		\icmlauthor{Chunyuan Li}{duke}
		\icmlauthor{Lawrence Carin}{duke}
	\end{icmlauthorlist}
	
	\icmlaffiliation{duke}{Duke University}
	\icmlaffiliation{bu}{SUNY at Buffalo}
	
	\icmlcorrespondingauthor{Changyou Chen}{cchangyou@gmail.com}
	
	\icmlkeywords{Machine Learning, ICML}
	
	\vskip 0.3in
	]
	
	
	
	\printAffiliationsAndNotice{} 
	
\begin{abstract}
	Policy optimization is a core component of reinforcement learning (RL), and most existing RL methods directly optimize parameters of a policy based on maximizing the expected total reward, or its surrogate. Though often achieving encouraging empirical success, its underlying mathematical principle on {\em policy-distribution} optimization is unclear. We place policy optimization into the space of {\em probability measures}, and interpret it as Wasserstein gradient flows. On the probability-measure space, under specified circumstances, policy optimization becomes a convex problem in terms of distribution optimization. To make optimization feasible, we develop efficient algorithms by numerically solving the corresponding discrete gradient flows. Our technique is applicable to several RL settings, and is related to many state-of-the-art policy-optimization algorithms. Empirical results verify the effectiveness of our framework, often obtaining better performance compared to related algorithms.
\end{abstract}
\vspace{-0.25in}
\section{Introduction}
\label{sec:intro}

\label{sec:intro}

There is recent renewed interest in reinforcement learning \cite{sutton1998reinforcement,kaelbling1996reinforcement}, largely as a consequence of the success of deep reinforcement learning (deep RL) \cite{mnih2015human,Li:arxiv17}, which is applicable to complex environments and has obtained state-of-the-art performance on several challenging problems. In reinforcement learning an agent interacts with the environment, seeking to learn an optimal policy that yields the maximum expected reward during the interaction. Generally speaking, a policy defines a distribution over actions conditioned on the states. Learning an optimal policy corresponds to searching for an element (a conditional action distribution) on the space of distributions that yields the best expected feedback (reward) to the agent as it interacts sequentially with the environment.

A standard technique for policy learning is the policy-gradient (PG) method \cite{SuttonMSM:NIPS00}. In PG, a policy is represented in terms of  parameters, typically optimized by stochastic gradient descent (SGD) to maximize the expected total reward. A similar idea has been applied for learning deterministic policies, termed deterministic policy gradient (DPG) \cite{silver2014deterministic}. Significant progress has been made on advancing policy learning since introduction of deep learning techniques. As examples, the deep deterministic policy gradient (DDPG) method combines DPG and $Q$-learning \cite{WatkinsD:MLJ92} to jointly learn a policy and a $Q$-function for continuous control problems \cite{LillicrapHPHETSW:ICLR16}. Trust region policy optimization (TRPO) improves PG by preserving the monotonic-policy-improvement property \cite{schulman2015trust}, implemented by imposing a trust-region constraint, defined as the Kullback-Leibler (KL) divergence between consecutive policies. Later, \cite{SchulmanWDRK:arxiv17} proposed proximal policy optimization (PPO) to improve TRPO by optimizing a ``surrogate'' objective with an adaptive KL penalty and reward-clipping mechanism.
Though obtaining encouraging empirical success, many of the aforementioned algorithms optimize parameters directly, and appear to lack a rigorous interpretation in terms of distribution optimization, {\it e.g.}, it is not mathematically clear how sequentially optimizing policy parameters based on an expected-total-reward objective corresponds to optimizing the {\em distribution} of policy itself. 

In this paper we introduce gradient flows in the space of probability distributions, called {\em Wasserstein gradient flows} (WGF), and formulate policy optimization in RL as a WGF problem. Essentially, WGF induces a geometry structure (manifold) in the distribution space characterized by an {\em energy functional}. The length between elements on the manifold is defined by the second-order Wasserstein distance. Thus, searching for an optimal distribution corresponds to traveling along a gradient flow on the space until convergence. In the context of deep RL, the energy functional is characterized by the expected reward. Gradient flow corresponds to a sequence of policy distributions converging to an optimal policy during an iterative optimization procedure. From this perspective, convergence behavior of the optimization can be better understood. 

Traditional stochastic policies are limited by their simple representation ability, such as using multinomial~\cite{mnih2015human} or Gaussian policy distributions~\cite{schulman2015trust}. To overcome this issue, the proposed WGF-based stochastic policies employ general energy-based representations. To optimize the stochastic policy, we define WGFs for RL in two settings: $\RN{1})$ indirect-policy learning, defined on a distribution space for {\em parameters}; $\RN{2})$ direct-policy learning, defined on a distribution space for {\em policy distributions}. These correspond to two variants of our algorithms. The original form of the WGF problem is hard to deal with,  as it is generally infeasible to directly optimize over a distribution (an infinite-dimensional object).
To overcome this issue, based on the Jordan-Kinderlehrer-Otto (JKO) method \cite{JordanKO:MA98}, we propose a particle-based algorithm to approximate a continuous density function with particles, and derive the corresponding gradient formulas for particle updates. Our method is conceptually simple and practically efficient, which also provides a theoretically sound way to use trust-region algorithms for RL. Empirical experiments show improved performance over related reinforcement learning algorithms.

\section{Preliminaries}\vspace{-0.1cm}
We review concepts and numerical algorithms for gradient flows. We start from gradient flows on the Euclidean space, and then extend them on the space of probability measures.
\vspace{-0.2cm}
\subsection{Gradient flows on the Euclidean space}\label{sec:gf_euc}
For a smooth function\footnote{We will focus on the convex case, since this is the case for many gradient flows on the space of probability measures, as detailed subsequently.} $F: \mathbb{R}^d \rightarrow \mathbb{R}$, and a starting point $\xb_0 \in \mathbb{R}^d$, the gradient flow of $F(\xb)$ is defined as the solution of the differential equation: $\frac{\mathrm{d}\xb}{\mathrm{d}\tau} = -\nabla F(\xb(\tau))$, for time $\tau > 0$ and initial condition $\xb(0) = \xb_0$. This is a standard Cauchy problem \cite{Rulla:NA96}, endowed with a unique solution if $\nabla F$ is Lipschitz continuous. When $F$ is non-differentiable, the gradient is replaced with its subgradient, which gives a similar definition, omitted for simplicity. 
\vspace{-0.4cm}
\paragraph{Numerical solution}
An exact solution to the above gradient-flow problem is typically intractable. A standard numerical method, called the {\em Minimizing Movement Scheme} (MMS) \cite{Gobbino:AMPA99}, evolves $\xb$ iteratively for small steps along the gradient of $F$ at the current point. Denoting the current point as $\xb_k$, the next point is $\xb_{k+1} = \xb_k - \nabla F(\xb_{k+1}) h$, with stepsize $h$. Note $\xb_{k+1}$ is equivalent to solving optimization problem
$\xb_{k+1} = \arg\min_{\xb}F(\xb) + \frac{\|\xb - \xb_k\|_2^2}{2h}$, 
where $\|\cdot\|_2$ denotes the vector 2-norm. Convergence of the $\{\xb_k\}$ sequence to the exact solution has been established \cite{Ambrosio:book05}. Refer to Section~\ref{sec:gf_euc1} of the Supplementary Material (SM) for details.
\vspace{-0.5cm}
\subsection{Gradient flows on the probability-measure space}\label{sec:wgf}
By placing the optimization onto the space of probability measures, denoted $\mathcal{P}(\Omega)$ with $\Omega \subset \mathbb{R}^d$, we arrive at Wasserstein gradient flows. For a formal definition, we first endow a Riemannian geometry on $\mathcal{P}(\Omega)$. The geometry is characterized by the length between two elements (two distributions), defined by the 2nd-order Wasserstein distance:
\vspace{-0.2cm}
{\small\begin{align*}
	W_2^2(\mu, \nu) \triangleq \inf_{\gamma}\left\{\int_{\Omega \times \Omega}\|\xb - \yb\|_2^2\mathrm{d}\gamma(\xb, \yb): \gamma \in \Gamma(\mu, \nu)\right\}~,
	\end{align*}}
\!\!where $\Gamma(\mu, \nu)$ is the set of joint distributions over $(\xb, \yb)$ such that the two marginals equal $\mu$ and $\nu$, respectively. This is an optimal-transport problem, where one wants to transform $\mu$ to $\nu$ with minimum cost \cite{Villani:08}. Thus the term $\|\xb - \yb\|_2^2$ represents the cost to transport $\xb$ in $\mu$ to $\yb$ in $\nu$, and can be replaced by a general metric $c(\xb, \yb)$ in a metric space. If $\mu$ is absolutely continuous w.r.t.\! the Lebesgue measure, there is a unique optimal transport plan from $\mu$ to $\nu$, {\it i.e.}, a mapping $T: \mathbb{R}^d \rightarrow\mathbb{R}^d$ pushing $\mu$ onto $\nu$ satisfying $T_{\#}\mu = \nu$. Here $T_{\#}\mu$ denotes the pushforward measure of $\mu$. The Wasserstein distance is equivalently reformulated as
\vspace{-0.2cm}
{\small\begin{align*}
	W_2^2(\mu, \nu) \triangleq \inf_{T}\left\{\int_{\Omega}c(\xb, T(\xb))\mathrm{d}\mu(\xb)\right\}~.
	\end{align*}}
\!\!Let $\{\mu_{\tau}\}_{\tau\in[0,1]}$ be an absolutely continuous curve in $\mathcal{P}(\Omega)$ with finite second-order moments. Consider $W_2^2(\mu_{\tau}, \mu_{\tau+h})$. Motivated by the Euclidean-space case, if we define $\vb_{\tau}(\xb) \triangleq \lim_{h\rightarrow 0}\frac{T(\xb_{\tau}) - \xb_{\tau}}{h}$ as the {\em velocity of the particle}, a gradient flow can be defined on $\mathcal{P}(\Omega)$ correspondingly \cite{Ambrosio:book05}.
\vspace{-0.2cm}
\begin{lemma}\label{theo:gf_w_exist}
	Let $\{\mu_{\tau}\}_{\tau\in[0,1]}$ be an absolutely-continuous curve in $\mathcal{P}(\Omega)$ with finite second-order 
	moments. Then for a.e.\! $\tau\in [0, 1]$, the above vector field $\vb_{\tau}$ defines a gradient flow on $\mathcal{P}(\Omega)$ as $\partial_{\tau} \mu_{\tau} + \nabla \cdot (\vb_{\tau} \mu_{\tau}) = 0$, where $\nabla\cdot\ab \triangleq \nabla_{\xb}^{\top} \ab$ for a vector $\ab$.
\end{lemma}
Function $F$ in Section~\ref{sec:gf_euc} is lifted to be a functional in the space of probability measures, mapping a probability measure $\mu$ to a real value, {\it i.e.}, $F: \mathcal{P}(\Omega) \rightarrow \mathbb{R}$. $F$ is the energy functional of a gradient flow on $\mathcal{P}(\Omega)$. Consequently, it can be shown that $\vb_{\tau}$ in Lemma~\ref{theo:gf_w_exist} has the form $\vb_{\tau} = -\nabla \frac{\delta F}{\delta \mu_{\tau}}(\mu_{\tau})$ \cite{Ambrosio:book05}, where $\frac{\delta F}{\delta \mu_{\tau}}$ is called the {\em first variation} of $F$ at $\mu_{\tau}$. Based on this, gradient flows on $\mathcal{P}(\Omega)$ can be written
\vspace{-0.3cm}
\begin{align}\label{eq:gf_dis1}
\partial_{\tau} \mu_{\tau} = -\nabla \cdot (\vb_{\tau} \mu_{\tau}) = \nabla \cdot \left(\mu_{\tau} \nabla(\frac{\delta F}{\delta \mu_{\tau}}(\mu_{\tau}))\right)~.
\end{align}

\begin{remark}
	Intuitively, an energy functional $F$ characterizes the landscape structure (appearance) of the corresponding manifold, and the gradient flow \eqref{eq:gf_dis1} defines a solution path on this manifold. Usually, by choosing appropriate $F$, the landscape is convex, {\it e.g.}, the It\'{o}-diffusion case defined below. This provides a theoretical guarantee of optimal convergence of a gradient flow.
\end{remark}

\paragraph{It\'{o} diffusions as WGFs}
It\'{o} diffusion defines a stochastic mapping $\mathcal{T}: \mathbb{R}^d \times \mathbb{R} \rightarrow \mathbb{R}^d$ such that we have $\mathcal{T}(\xb, 0) = \xb$ and $\mathcal{T}(\mathcal{T}(\xb, \tau), s) = \mathcal{T}(\xb, s + \tau)$, for all $\xb \in \mathbb{R}^d$ and $s, \tau \in \mathbb{R}$. A typical example of this family is defined as $\mathcal{T}(\xb, \tau) = \xb_{\tau}$, where $\xb_{\tau}$ is driven by a diffusion of the form:
\vspace{-0.4cm}
\begin{align}\label{eq:diffusion}
\mathrm{d}\xb_{\tau} = \nabla U(\xb_{\tau}) \mathrm{d}\tau + \sigma(\xb_{\tau})\mathrm{d}\mathcal{W}~.
\end{align}
Here $U: \mathbb{R}^d\rightarrow\mathbb{R}^d$, $\sigma: \mathbb{R}^{d} \rightarrow \mathbb{R}^{d\times d}$ are called the drift and diffusion terms, respectively; $\mathcal{W}$ is the standard $d$-dimensional Brownian motion. In Bayesian inference, we seek to make the stationary distribution of $\xb_{\tau}$ approach a particular distribution $p(\xb)$, {\it e.g.}, a posterior distribution. One solution for this is to set $U(\xb_{\tau}) = \frac{1}{2}\log p(\xb)$ and $\sigma(\xb_{\tau})$ as the $d\times d$ identity matrix. The resulting diffusion is called Langevin dynamics \cite{WellingT:ICML11}. Denoting the distribution of $\xb_{\tau}$ as $\mu_{\tau}$, it is well known \cite{Risken:FPE89} that $\mu_{\tau}$ is characterized by the Fokker-Planck (FP) equation: 
\begin{align}\label{eq:FPE}
\partial_{\tau} \mu_{\tau} = \nabla\cdot \left(-\mu_{\tau}\nabla U + \nabla\cdot\left(\mu_{\tau}\sigma\sigma^{\top}\right)\right)~.
\end{align}
Note \eqref{eq:FPE} is in the gradient-flow form of \eqref{eq:gf_dis1}, where the energy functional $F$ is defined as\footnote{We assume $\sigma$ to be independent of $\xb$, which is the case in Langevin dynamics whose stationary distribution is set to be proportional to $e^{-U(\xb)}$. As a result, we drop $\sigma$ in the following.}:
{\small\begin{align}\label{eq:ito_energy}
	F(\mu)\triangleq \underbrace{-\int U(\xb)\mu(\xb)\mathrm{d}\xb}_{F_1} + \underbrace{\int \mu(\xb)\log\mu(\xb)\mathrm{d}\xb}_{F_2}
	\end{align}}
Note $F_2$ is the energy functional of a pure Brownian motion ({\it e.g.}, $U(\xb) = 0$ in \eqref{eq:diffusion}). To verify the FP equation with \eqref{eq:gf_dis1}, the first variation of $F_1$ and $F_2$ is calculated as
\begin{align}\label{eq:firstvariation}
\frac{\delta F_1}{\delta \mu} = -U,~~~~\frac{\delta F_2}{\delta \mu} = \log \mu + 1~.
\end{align}
Substituting \eqref{eq:firstvariation} into \eqref{eq:gf_dis1} recovers the FP equation \eqref{eq:FPE}.

\paragraph{Numerical methods}
Inspired by the Euclidean-space case, gradient flow \eqref{eq:gf_dis1} can be approximately solved by discretizing time, leading to an iterative optimization problem, where for iteration $k$: $\mu_{k+1}^{(h)} \in \arg\min_{\mu}F(\mu) + \frac{W_2^2(\mu, \mu_k^{(h)})}{2h}$.
Specifically, for It\'{o}-diffusion with $F$ defined in \eqref{eq:ito_energy}, the optimization problem becomes:
\begin{align}\label{eq:ito_discrete}
\mu_{k+1}^{(h)} = \arg\min_{\mu}\KL\left(\mu\|p(\xb)\right) + \frac{W_2^2(\mu, \mu_k^{(h)})}{2h}~,
\end{align}
where $p(\xb) \triangleq \frac{1}{Z}e^{U(\xb)}$ is the target distribution. This procedure is called the Jordan-Kinderlehrer-Otto (JKO) scheme. Remarkably, the convergence of \eqref{eq:ito_discrete} can be guaranteed \cite{JordanKO:MA98}, as stated in Lemma~\ref{theo:variational_fp}.

\begin{lemma}\label{theo:variational_fp}
	Assume that $\log p(\xb)\leq C_1$ is infinitely differentiable, and $\|\nabla \log p(\xb)\| \leq C_2\left(1 + C_1 - \log p(\xb)\right) (\forall \xb)$ for some constants $\{C_1, C_2\}$. Let $T = hK$, $\mu_0 \triangleq q_0(\xb)$, and $\{\mu_k^{(h)}\}_{k=1}^K$ be the solution of the functional optimization problem \eqref{eq:ito_discrete}, which are restricted to the space with finite second-order moments. Then $\RN{1})$ the problem \eqref{eq:ito_discrete} is convex; and $\RN{2})$ $\mu_K^{(h)}$ converges to $\mu_T$ in the limit of $h\rightarrow 0$, {\it i.e.}, $\lim_{h\rightarrow 0}\mu_K^{(h)} = \mu_T$, where $\mu_T$ is the solution of \eqref{eq:FPE} at $T$.
\end{lemma}

\begin{remark}
	Since the stationary distribution of the FP equation \eqref{eq:FPE} is proportional to $e^{U(\xb)}$, Lemma~\ref{theo:variational_fp} suggests that $\lim_{k\rightarrow\infty, h\rightarrow 0} \mu_{k}^{(h)}= \frac{1}{Z}e^{U}$, a useful property to guide design of energy functionals for RL, as discussed in Section~\ref{sec:RLWGF}.
\end{remark}

\section{Particle Approximation for WGFs}\label{sec:par_opt}
We focus on solving It\'{o} diffusions with scheme \eqref{eq:ito_discrete}. Directly reformulating gradient flows as a sequential optimization problem in \eqref{eq:ito_discrete} is infeasible, because $\{\mu_k^{(h)}\}$ are infinite-dimensional objects. We propose to use particle approximation to solve \eqref{eq:ito_discrete}, where particles continuously evolve over time. There exist particle-based algorithms for gradient-flow approximations, for example, the stochastic and deterministic particle methods in \cite{CottetK:00,Russo:CPAM90,CarrilloCP:arxiv17}. However, they did not target the JKO scheme, and thus are not applicable to our setting. Another advantage of the JKO scheme is that it allows direct application of gradient-based algorithms, once we get gradients of the particles; thus, it is particularly useful in deep-learning-based methods where parameters are updated by backpropagating gradients through a network.

Following similar idea as in \cite{ChenZWLC:tech18}, in the $k$-th iteration of our algorithm, $M$ particles $\{\xb_k^{(i)}\}_{i=1}^M$ are used to approximate $\mu_k^{(h)}$:
$\mu_k^{(h)} \approx \frac{1}{M}\sum_{i=1}^M\delta_{\xb_k^{(i)}}$.
Our goal is to evolve $\{\xb_k^{(i)}\}$ such that the corresponding empirical measure, $\mu^{(h)} \triangleq \frac{1}{M}\sum_{i=1}^M\delta_{\xb^{(i)}}$, minimizes \eqref{eq:ito_discrete}. A standard method is to use gradient descent to update the particles, where gradients $\{\frac{\partial \KL(\mu, \mu_k^{(h)})}{\partial \xb^{(i)}}, \frac{\partial W_2^2(\mu, \mu_k^{(h)})}{\partial \xb^{(i)}}\}$ are required according to \eqref{eq:ito_discrete}. By assuming $\xb_k^{(i)}$ to evolve in the form of $\xb_{k+1}^{(i)} = \xb_k^{(i)} + h \phi(\xb_k^{(i)})$, with function $\phi$ restricted to an RKHS with kernel $K(\cdot, \cdot)$, the gradient of the KL term is calculated as \cite{LiuW:NIPS16}:
{\small\begin{align}\label{eq:klgrad}
	\frac{\partial \KL(\mu^{(h)}, \mu_k^{(h)})}{\partial \xb^{(i)}} \propto &\frac{1}{M}\sum_{j=1}^M\left[K(\xb^{(j)}, \xb^{(i)})\nabla_{\xb^{(j)}}\log p(\xb^{(j)})\right.\nonumber\\ 
	&\left.+ \nabla_{\xb^{(j)}}K(\xb^{(j)}, \xb^{(i)})\right]~.
	\end{align}}
The gradient for $W_2^2(\mu^{(h)}, \mu_k^{(h)})$ is more involved, as the distance does not have a closed form. The Wasserstein term arises due to the Brownian motion in the diffusion process \eqref{eq:diffusion}, and the non-differentiability of a sample path from a Brownian motion is translated into the Wasserstein distance. We develop a simple yet effective method to overcome this issue below.

First, using a particle approximation, $W_2^2(\mu^{(h)}, \mu_k^{(h)})$ is simplified as
\vspace{-0.3cm}{\small\begin{align}\label{eq:w2}
	W_2^2&(\mu, \mu_k^{(h)}) = \inf_{p_{i,j}}\sum_{i,j}p_{ij}c(\xb^{(i)}, \xb_k^{(j)}) \\
	s.t.&~~ \sum_j p_{ij} = \frac{1}{M}, ~~\sum_i p_{ij} = \frac{1}{M}~,\nonumber
	\end{align}}
\vspace{-0mm}
where $c(\xb_1, \xb_2) \triangleq \|\xb_1 - \xb_2\|_2^2$. Our goal turns to solving for the optimal $\{p_{ij}\}$. Since $W_2$ comes from the Brownian motion, the energy functional in its gradient flow is defined as $F_2$ in \eqref{eq:ito_energy}. Solving the gradient flow with the JKO scheme, at each iteration we minimize $\lambda F_2 + W_2^2(\mu, \mu_k^{(h)})$ with $\lambda$ a regularization parameter. Substituting $W_2^2$ with \eqref{eq:w2}, introducing Lagrangian multipliers $\{\alpha_i, \beta_i\}$ to deal with the constraints, and letting $c_{ij} \triangleq c(\xb^{(i)}, \xb_k^{(j)}) $, the dual problem is,
\vspace{-0.35cm}
\begin{align*}
\mathcal{L}&(\{p_{ij}\}, \{\alpha_i\}, \{\beta_i\}) = \lambda\sum_{i,j}p_{ij}\log p_{ij} + p_{ij}c_{ij} \\
&+ \sum_i \alpha_i(\sum_j p_{ij} - \frac{1}{M}) + \sum_j \beta_j(\sum_i p_{ij} - \frac{1}{M})
\end{align*}
\vspace{-5.5mm}

The optimal $p_{ij}$ have forms of $p_{ij}^* = u_i e^{-c_{ij}/\lambda}v_j$, where $u_i \triangleq e^{-\frac{1}{2}-\frac{\alpha_i}{\lambda}}$, $v_j = e^{-\frac{1}{2}-\frac{\beta_j}{\lambda}}$.
Assuming $\{u_i\}$ and $\{v_j\}$ are independent of $\{\xb^{(i)}\}$ and $\{\xb_k^{(j)}\}$, 
\vspace{-2mm}
\begin{align}\label{eq:w2grad}
&\frac{\partial W_2^2(\mu, \mu_k^{(h)})}{\partial \xb^{(i)}} \propto \frac{\sum_jc_{ij}e^{-c_{ij}/\lambda}}{\partial \xb^{(i)}} \nonumber\\
=& \sum_j 2\left(1 - \frac{c_{ij}}{\lambda}\right)e^{-c_{ij}/\lambda}(\xb^{(i)} - \xb_k^{(j)})~.
\end{align}\vspace{-7mm}

The gradients of particles can be obtained by combining \eqref{eq:klgrad} and \eqref{eq:w2grad}, which are then optimized using SGD. Intuitively, from \eqref{eq:w2grad}, the Wasserstein term contributes in two ways: $\RN{1})$ When $\frac{c_{ij}}{\lambda} > 1$, $\xb^{(i)}$ is pulled close to previous particles $\{\xb_k^{(j)}\}$, with force proportional to $(\frac{c_{ij}}{\lambda} - 1)e^{-c_{ij}/\lambda}$; $\RN{2})$ when $\xb^{(i)}$ is close enough to a previous particle $\xb_k^{(j)}$, {\it i.e.}, $\frac{c_{ij}}{\lambda} < 1$, $\xb^{(i)}$ is pushed away, preventing it from collapsing to $\xb_k^{(j)}$.
\vspace{-0.2cm}
\section{Policy Optimization as WGFs}\label{sec:RLWGF}
Reinforcement learning is the problem of finding an optimal policy for an agent interacting with an unknown environment, collecting a reward per action. A policy is defined as a conditional distribution, $\pi(\ab|\sbb)$, defining the probability over an action $\ab\in \mathcal{A}$ conditioned on a state variable $\sbb\in\mathcal{S}$. Formally, the problem can be described as a Markov decision process (MDP), $\mathcal{M} = \langle\mathcal{S}, \mathcal{A}, P_s, r, \gamma\rangle$, where $P_s(\sbb^\prime|\sbb, \ab)$ is the transition probability from state $\sbb$ to $\sbb^\prime$ given action $\ab$; $r(\sbb, \ab)$ is an unknown reward function immediately following the action $\ab$ performed at state $\sbb$; $\gamma \in [0, 1]$ is a discount factor regularizing future rewards. We denote these variables with a subscript $t$ to indicate their time dependency. At each time step $t$, conditioned on the current state $\sbb_t$, the agent chooses an action $\ab_t \sim \pi(\cdot|\sbb_t)$ and receives a reward signal\footnote{We assume the reward function to be deterministic, for simplicity; stochastic rewards can be addressed similarly.} $r(\ab, \sbb)$. The environment as seen by the agent then updates its state as $\sbb_{t+1} \sim P_s(\cdot|\sbb_t, \ab_t)$. The goal is to learn an optimal policy such that one obtains the maximum expected total reward, {\it e.g.}, by maximizing
\vspace{-2mm}
{\small\begin{align}\label{eq:policy_learning}
	J(\pi) = \sum_{t=1}^{\infty}\mathbb{E}_{P_s, \pi}\left[\gamma^tr(\ab, \sbb)\right] = \mathbb{E}_{\sbb\sim\rho_{\pi}, \ab\sim\pi}\left[r(\sbb, \ab)\right]
	\end{align}}
\vspace{-5mm}

where $\rho_{\pi} \triangleq \sum_{t=1}^{\infty}\gamma^{t-1}P_r(\sbb_t = \sbb)$, and $P_r(\sbb)$ denotes the state marginal distribution induced by $\pi$. 
Optimizing the objective in (\ref{eq:policy_learning}) with a maximum entropy constraint provides us with a framework for training stochastic policies, where specific forms of these policy distribution are required, restricting the representation power. To define a more general class of distributions that can represent more complex and multimodal distributions, we adopt the general energy-based policies \cite{HaarnojaTAL:ICML17}, and transform it into the WGF framework.
	
	Specifically, in the WGF framework, policies form a Riemannian manifold on the space of probability measures. The manifold structure is determined by the expected total reward \eqref{eq:policy_learning}, and the geodesic length between two elements (policy distributions) is defined as the standard second-order Wasserstein distance. With convex energy functionals (defined below), searching for an optimal policy reduces to running SGD on the manifold of probability measures.
	
	In the following, we define gradient flows on both parameter-distribution space and policy-distribution space, leading to indirect-policy learning and direct-policy learning, respectively. In indirect-policy learning, a WGF is defined over policy parameters; whereas in direct-policy learning, a WGF is defined over actions. For both settings, different energy functionals are defined based on the expected total reward, as detailed below. We note that most existing deep RL algorithms cannot be included into the two settings, without the concept of WGF. However, their specific techniques could be applied as intermediate ingredients in our framework.
	\vspace{-0.3cm}
	\subsection{Indirect-policy learning}\label{sec:impPL}
	\vspace{-0.2cm}
	
	With indirect-policy learning, we do not optimize the stochastic policy $\pi$ directly. Instead, we aim to describe uncertainty of a policy with parameter distributions (weight uncertainty). Thus we define a gradient flow on the parameters. Let a policy be parameterized by $\thetab$, denoted as $\pi_{\thetab}$. If we treat $\thetab$ as stochastic and learn  its posterior distribution $p(\thetav)$ in response to the expected total reward, the policy is implicitly learned in the sense that uncertainty in the parameter is transferred into the policy distribution in prediction. Following \citep{houthooft2016vime, liu2017stein}, the objective function is defined as:
	\begin{equation}
	\label{equ:rlobjective}
	\begin{aligned}
	\max_{p}\{\mathbb{E}_{p(\thetav)}[J(\pi_\thetav)] - \alpha \KL(p\|p_0)\}
	\end{aligned}
	\end{equation}
	where $p_0(\thetav)$ is the prior of  $\thetav$; $\alpha \in [0, +\infty)$ is the temperature hyper-parameter to balance exploitation and exploration in the policy.
	If we use an uninformative prior, $p_0(\theta) = $ const, the KL term is simplified to the entropy as $\max_{p}\{\mathbb{E}_{p(\thetav)}[J(\pi_\thetav)] + \alpha \mathcal{H}(p)\}$.
	By taking the derivative of the objective function, the optimal distribution is shown to have a simple closed form of $p(\thetav)\propto \exp\left(J(\pi_\thetav)/\alpha\right)$ \citep{liu2017stein}.
	%
	This formulation is equivalent to a Bayesian formulation of parameter $\thetav$, where $p(\thetav)$ can be seen as the ``posterior'' distribution, and $\exp(J(\pi_\thetav)/\alpha)$ is the ``likelihood'' function.
	A variational (posterior) distribution for $\thetab$, denoted as $\mu(\thetab)$, is learned by solving an appropriate gradient-flow problem.  We define an energy functional characterizing the similarity between the current parameter distribution and the true distribution induced by the total reward as
	\vspace{-0.3cm}
	{\small
	\begin{align}\label{eq:imp_energy}
	F(\mu) &\triangleq -\int J(\pi_{\thetab})\mu(\thetab)\mathrm{d}\thetab + \int \mu(\thetab)\log\mu(\thetab)\mathrm{d}\thetab \nonumber\\
	&= \KL\left(\mu\|p_{\thetab}\right)~,
	\end{align}}
	\!\!The energy functional defines a landscape determined by the expected total reward, and obtains its minimum when $\mu = p_{\thetab}$.
	\begin{proposition}\label{prop:IPWGF}
		For the gradient flow with energy functional defined in \eqref{eq:imp_energy}, $\mu$ converges to $p_{\thetab}$ in the infinite-time limit.
	\end{proposition}
	To solve the above gradient-flow problem, one can apply the JKO scheme with a stepsize $h$ (we follow previous notation to use subscript $k$ to denote discrete-time solutions and superscript $h$ to denote the stepsize):
	{\small
	\begin{align}\label{eq:imp_discrete}
	\mu_{k+1}^{(h)} = \arg\min_{\mu}\KL\left(\mu\|p_{\thetab}\right) + \frac{W_2^2(\mu, \mu_k^{(h)})}{2h}~.
	\end{align}}
	\!\!The above problem can be directly solved with gradient descent by adopting the particle approximation described in Section 3. Specifically, let the current particles be $(\thetab^{(i)})_{i=1}^M$. When calculating $\frac{\partial \KL(\mu\|p_{\thetab})}{\partial \thetab^{(i)}}$ as in \eqref{eq:klgrad}, we need to evaluate $\nabla_{\thetab^{(i)}}J(\pi_{\thetab^{(i)}})$. This can be approximated with REINFORCE \cite{williams1992simple} or advantage actor critic \cite{schulman2015high}. For example, with REINFORCE,
	{\small
	\begin{align*}
	\nabla_{\thetab^{(i)}}J(\pi_{\thetab^{(i)}}) \approx \frac{1}{T}\sum_{t=1}^T\gamma^{t-1}\nabla_{\thetab^{(i)}}\log \pi_{\thetab^{(i)}}(\ab_t|\sbb_t)\hat{Q}^{\pi}(\sbb_t, \ab_t)
	\end{align*}}
	\!\!where $T$ is a horizon parameter, and $\hat{Q}^{\pi}(\sbb_t, \ab_t)$ is the Q-value function. We call this variant of our framework Indirect Policy learning with WGF (IP-WGF).

	\begin{remark}
		Assume gradients $\nabla_{\thetab^{(i)}}J(\pi_{\thetab^{(i)}})$ and $\nabla_{\thetab^{(i)}}W_2^2$ are unbiased. Under the limit of $M\rightarrow \infty$ and $h\rightarrow 0$, and based on the fact that $F$ in \eqref{eq:imp_energy} is convex, Lemma~\ref{theo:variational_fp} suggests the particle approximation converges to the global minimum $p_{\thetab}$. The conclusion applies, in the next section, similarly in the direct-policy-learning case. Existing methods such as TRPO \cite{schulman2015trust} and PPO \cite{SchulmanWDRK:arxiv17} do not have such an interpretation, thus understanding their underlying convergence is more challenging. Furthermore, these methods optimize parameters directly as fixed points, deteriorating their ability to explore when policy distributions are inappropriately defined, as stochasticity only comes from the policy distributions.
	\end{remark}
	
	\subsection{Direct-policy learning}\label{sec:exp_learning}
	When the dimension of parameter space is high, as is often the case in practice, IP-WGF can suffer from computation and storage inefficiencies.
	In direct-policy learning, a gradient flow is defined for the distribution of policies, thus a policy is {\em directly} optimized during learning. This approach appears to be more efficient and flexible, and connects more directly to existing works compared with indirect-policy learning.
	
	Specifically, we consider a general energy-based policies of the form $\pi(\ab|\sbb)\propto \exp(-\varepsilon(\sbb,\ab)/\alpha)$ that is able to model more complex distributions \cite{HaarnojaTAL:ICML17}. We formulate the direct-policy learning as policy-distribution-based gradient flows. The energy functional is defined with respect to the learned policy $\pi$, thus it should depend on states. To this end, let $\varepsilon_{s, \pi}(\ab) = -Q(\ab_t, \sbb_t)$, where $Q(\ab_t, \sbb_t)\triangleq r(\ab_t = \ab, \sbb_t = \sbb) + \mathbb{E}_{(\sbb_{t+1}, \ab_{t+1}, \cdots) \sim (\rho_{\pi}, \pi)}\sum_{l=1}	^{\infty}\gamma^lr(\sbb_{t+l}, \ab_{t+l})$. $Q(\ab_t, \sbb_t)$ is seen to be a functional depending on the current $\sbb_t$ and $\ab_t$, as well as the policy $\pi$. Integrating out the action $\ab$, an energy functional characterizing the similarity of the current policy $\pi$ and the optimal policy, $p_{s,\pi}(\ab|\sbb) \propto e^{Q(\ab, \sbb)}$, is readily defined as
	{\small\begin{align}\label{eq:exp_energy}
		F_s(\pi) &\triangleq -\int Q(\ab, \sbb)\pi(\ab|\sbb)\mathrm{d}\ab + \int \pi(\ab|\sbb)\log\pi(\ab|\sbb)\mathrm{d}\ab \nonumber\\
		&= \KL\left(\pi\|p_{s,\pi}\right)~.
		\end{align}}
	\vspace{-0.6cm}
	\begin{remark}
		Soft $Q$-learning \cite{HaarnojaTAL:ICML17} adopts $Q(\ab, \sbb) + \mathcal{H}(\pi)$ as the objective function, where the entropy of $\pi$, $\mathcal{H}(\pi)\triangleq-\mathbb{E}_{\pi}[\log\pi]$, is included to add stochasticity into the corresponding $Q$-function. By treating the problem as a WGF, the stochasticity is modeled directly in the policy distribution, thus we do not include the entropy term, though it is of no harm to add it in.
	\end{remark}
	\vspace{-0.2cm}
	\begin{proposition}\label{prop:dpwgf}
		For a WGF with the energy functional defined in \eqref{eq:exp_energy}, $\pi(\ab|\sbb)$ converges to $p_{s, \pi}(\ab) \propto e^{Q(\ab, \sbb)}$ with $Q(\ab, \sbb)$ satisfying the following modified Bellman equation:
		{\small\begin{align*}
			Q(\ab_t, \sbb_t) = r(\ab_t, \sbb_t) + \gamma \mathbb{E}_{\sbb_{t+1}\sim \rho_{\pi}} [V_{\pi}(\sbb_{t+1}) - \mathcal{H}(\pi(\cdot|\sbb_{t+1}))]
		\end{align*}}
		where $V_{\pi}(\sbb_{t+1}) \triangleq \log\int_{\mathcal{A}}\exp(Q(\ab, \sbb_{t+1}))\mathrm{d}\ab$.
	\end{proposition}
	
To solve the corresponding WGF, we again adopt the JKO scheme as in \eqref{eq:imp_discrete} to optimize the policy $\pi$ by particle approximation, {\it i.e.}, $\pi \propto \frac{1}{M}\sum_{i=1}^M\delta_{\ab^{(i)}}$. One challenge is that when calculating $\frac{\partial \KL(\pi\|p_{s,\pi})}{\partial \ab^{(i)}}$, from \eqref{eq:klgrad}, one needs to evaluate $p_{s,\pi}(\ab^{(i)})$, which is difficult due to the infinite time horizon and the unknown reward function $r(\sbb, \ab)$ when calculating $Q(\ab^{(i)}, \sbb)$. To address this, we approximate the \textbf{soft} $Q$-function, $Q(\cdot, \sbb)$, with a deep neural network $Q_s^{\thetab}(\sbb, \ab)$ parametrized by $\thetab$, {\it i.e.}, $p_{s,\pi}(\ab) \propto e^{Q_s^{\thetab}(\sbb, \ab)}$. The neural network $Q_s^{\thetab}$ naturally leads to a soft approximation of the standard $Q$-function according to Proposition~\ref{prop:dpwgf}. As a result, the learning can be done by alternating between the following two steps.

	
\paragraph{1) Optimizing the policy}
Given $Q_s^{\thetab}$, we could adopt the particle approximation with the JKO scheme to optimize the policy. However, since a policy is a conditional distribution, one needs to introduce a set of particles for each state. For large or continuous state space, this becomes intractable. To mitigate this problem, we propose to use a stochastic state-conditioned neural network $f^{\phib}$ parametrized by $\phib$ to approximate the policy. We call such a network a {\em sampling network}. 


The input to $f^{\phib}$ is a concatenation of a state $\sbb$ and a random-noise sample $\xi$ drawn from a simple distribution, {\it e.g.}, the standard normal distribution. To optimize the sampling network, note that the JKO scheme, with energy functional \eqref{eq:exp_energy}, is written as (we rewrite $\pi$ as $\pi^{\phib}$ to explicitly indicate the dependence of $\pi$ on $\phib$):
\vspace{-0.2cm}
{\small
	\begin{equation}\label{eq:exp_discrete}
	\begin{aligned}
	\pi_{k+1}^{\phib} = \arg\min_{\pi^{\phib}}\KL\left(\pi^{\phib}\|p_{s,\pi}\right) + \frac{W_2^2(\pi^{\phib}, \pi_k^{\phib})}{2h} \triangleq J_{\pi}^{\phib}~.
	\end{aligned}
	\end{equation}
}
\!\!The outputs of $f^{\phib}(\{\xi_i\}; \sbb_t)$ are particles $(\ab_t^{(i)})_{i=1}^M$. Using chain rule we calculate the gradient of $\phib$ as
{\small
	\begin{align*}
	\frac{\partial J_{\pi}^{\phib}}{\partial \phib} = \mathbb{E}_{\{\xi_i\}}\left[\frac{\partial J_{\pi}^{\phib}}{\partial \ab_t^{(i)}}\frac{\partial \ab_t^{(i)}}{\partial \phib}\right]~.
	\end{align*}
}
\!\!Thus $\phib$ can be updated using standard SGD, where $\partial J_{\pi}^{\phib}/\partial \ab_t^{(i)}$ represents particle gradients in the WGF, and is approximated using techniques from Section~\ref{sec:par_opt}; $\partial \ab_t^{(i)}/\partial \phib$ can be calculated by standard backpropagation.

\paragraph{2) Optimizing the $Q$-network·}
We optimize the $Q$-network using the Bellman error as in the soft-$Q$ learning setting \citep{HaarnojaTAL:ICML17}. Specifically, in each iteration, we optimize the following objective function:
{\small
	\begin{align*}
	J_Q(\thetab) \triangleq \mathbb{E}_{\sbb_t \sim q_{\sbb_t}, \ab_t\sim q_{\ab_t}}\left[\frac{1}{2}\left(\hat{Q}_s^{\bar{\thetab}}(\sbb_t, \ab_t) - Q_s^{\thetab}(\sbb_t, \ab_t)\right)^2\right]~,
	\end{align*}
}
\!\!\!\!where $q_{\sbb_t}$ and $q_{\ab_t}$ are arbitrary distributions with support on $\mathcal{S}$ and $\mathcal{A}$, respectively; $\hat{Q}_s^{\bar{\thetab}}(\sbb_t, \ab_t) = r(\sbb_t, \ab_t) + \gamma\mathbb{E}_{\sbb_{t+1}\sim \rho_{\pi}}[V_s^{\bar{\thetab}}(\sbb_{t+1})]$ is the target $Q$-value, with $V_s^{\bar{\thetab}}(\sbb_{t+1}) = \log\mathbb{E}_{q_{\ab^\prime}}[\frac{\exp(Q_s^{\bar{\thetab}}(\sbb_{t+1}, \ab^\prime))}{q_{\ab^\prime}(\ab^\prime)}] - \mathcal{H}(q_{\ab^\prime})$; $\bar{\thetab}$ represents the parameters of the target Q-network, as used in standard deep $Q$-learning \cite{mnih2013playing}. 
$q_{\ab_t}$ can be set to the distribution induced by the sampling network $f^{\phib}$ \cite{HaarnojaTAL:ICML17}. Alternatively, the form of $q_{\ab_t}$ can be explicitly defined, {\it e.g.}, using isotropic Gaussian or mixture of Gaussian distributions. The full algorithm is given in Section~\ref{supp:alg} of the SM. We call this variant of our framework Direct Policy learning with WGF (DP-WGF).

\paragraph{Reducing Variance}
Note that when optimizing the $Q$-network, one needs to calculate the $V_s^{\thetab}$-function. This includes an integration over $q_{\ab^{\prime}}$, which endows the high variance associated with Monte Carlo integration. Consequently, we propose to learn a $V$-network to approximate $V_s^{\thetab}$, denoted as $\bar{V}_{\psib}(\sbb)$ with parameter $\psib$. To learn the $V$-network, similar to \cite{HaarnojaZAL:arxiv18}, we use an explicit policy distribution. As a result, we replace the sampling network $f^{\phib}$ with a BNN discussed in Section~\ref{sec:impPL}, whose induced policy distribution is denoted $\tilde{\pi}_{\phib}$. Intuitively, $V_s^{\thetab}(\sbb)$ can be considered an approximation to the $\log$-normalizer of $\exp(Q_s^{\thetab}(\sbb, \ab))$ over $\ab$. From the definition, the objective is defined as: {\small $J_V(\psib) \triangleq \mathbb{E}_{\sbb_t \sim q_{\sbb_t}}\big(\bar{V}_{\psib}(\sbb_t) - \log \mathbb{E}_{\ab_t \sim \tilde{\pi}_{\phib}(\sbb_t)}\exp \big(Q_s^{\thetab}(\sbb_t, \ab_t) / \tilde{\pi}_{\phib}(\ab_t|\sbb_t)\big) $} {\small $ - \mathbb{E}_{\ab_t \sim \tilde{\pi}_{\phib}(\sbb_t)}\log \tilde{\pi}_{\phib}(\ab_t|\sbb_t)\big)^2$}. 
In our implementation, we find the following approximation works well: {\small $J_V(\psib) \triangleq \mathbb{E}_{\sbb_t \sim q_{\sbb_t}}\left(\bar{V}_{\psib}(\sbb_t) - \mathbb{E}_{\ab_t \sim \tilde{\pi}_{\phib}(\sbb_t)}[Q_s^{\thetab}(\sbb_t, \ab_t) - \log \tilde{\pi}_{\phib}(\ab_t|\sbb_t)]\right)^2$},  which is inspired by \cite{HaarnojaZAL:arxiv18}. We call this variant Direct Policy learning with WGF and Variance reduction (DP-WGF-V).

	\section{Connections with Related Works}\label{sec:related}
	\vspace{-0.1cm}
	\paragraph{Soft-$Q$ learning}
	Though motivated from different perspectives, our DP-WGF results in a similar algorithm as soft-$Q$ learning with energy based policies \cite{HaarnojaTAL:ICML17}. However, DP-WGF is more general, in that we can define different sampling networks, such as a BNN, which can be optimized with the proposed IP-WGF technique.
	\vspace{-0.3cm}
	\paragraph{Soft actor-critic (SAC)}
	SAC \cite{HaarnojaZAL:arxiv18} is an improvement of soft-$Q$ learning, introducing a similar $V$-network as ours and modeling the policy with a mixture of Gaussians. DP-WGF-V is related to SAC, but with a $V$-network from a different perspective (variance reduction). Importantly, we define a gradient flow for policy distributions, allowing optimization from a distribution perspective.
	\vspace{-0.7cm}
	\paragraph{Trust-region methods}
	Trust-Region methods are known to stabilize policy optimization in RL \cite{schulman2015trust}. 
	\citet{schulman2017equivalence} illustrate the equivalence between soft $Q$-learning and policy gradient. In the original TRPO setting, an objective function is optimized subject to a constraint that the updated policy is not too far from the current policy, in terms of the KL divergence (see Section~G for a more detailed descriptions). The theory of TRPO suggests adding a penalty to the objective instead of adopting a constraint, resulting in a similar form as our framework. However, we use the Wasserstein distance to penalize the previous and current policies, which is a weaker metric than the KL divergence, and potentially leads to more robust solutions. This is evidenced by the development of Wasserstein GAN \cite{ArjovskyCB:arxiv17}. As a result, our framework can be regarded as a trust-region-based counterpart for solving the soft Q-learning~\cite{HaarnojaTAL:ICML17} and SVPG~\cite{liu2017stein}. Similar arguments hold for other trust-region methods such as PPO \cite{SchulmanWDRK:arxiv17} and Trust-PCL \cite{NachumNXS:arxiv17}, which improve TRPO with either different objective or trust-region constraints.
	\vspace{-0.7cm}
	\paragraph{Noisy exploration}
	Adding noise to the parameters for noisy exploration \cite{fortunato2017noisy, plappert2017parameter} can be interpreted as a special case of our IP-WGF framework with a single particle. Isotropic Gaussian noisy exploration corresponds to the maximum {\em a posterior} (MAP) solution with a Gaussian assumption on the posterior distributions of parameters, potentially leading to inferior solutions when the assumption is not met. By contrast, our method is endowed with the ability to explore multimodal distributions, by optimizing the parameter distribution directly. More details are provided in Section~\ref{supp:related} of the SM.
	
	%
	%

		\begin{table}[tb]
	\centering
	\begin{adjustbox}{scale=0.80,tabular=cccc,center}
		\toprule[1.2pt]
		Dataset & PBP & SVGD& WGF \\
		\midrule
		Boston & -2.57 $\pm$ 0.09&-2.50$\pm$0.03&$\mathbf{ -2.40 \pm 0.10 }$\\
		Concrete &-3.16 $\pm$ 0.02&-3.08$\pm$0.02&$\mathbf{ -2.95 \pm 0.06 }$\\
		Energy &-2.04 $\pm$ 0.02&-1.77$\pm$0.02&$\mathbf{ -0.73 \pm 0.08 }$\\
		Kin8nm&0.90 $\pm$ 0.01&$\mathbf{ 0.98\pm0.01}$& 0.97 $\pm$ 0.02 \\
		Naval&3.73 $\pm$ 0.01&4.09$\pm$0.01&$\mathbf{ 4.11 \pm 0.02 }$\\
		CCPP&-2.80 $\pm$ 0.05&-2.82$\pm0.01$&$\mathbf{ -2.78 \pm 0.01 }$\\
		Winequality &-0.97 $\pm$ 0.01& -0.93$\pm$0.01&$\mathbf{-0.87 \pm 0.04}$ \\
		Yacht&-1.63 $\pm$ 0.02&-1.23$\pm$0.04&$\mathbf{ -0.99 \pm 0.15 }$\\
		Protein& -2.97 $\pm$ 0.00& -2.95$\pm$0.00& $\mathbf{-2.88 \pm 0.01}$\\
		YearPredict&-3.60$\pm$NA&-3.58 $\pm$ NA & $\mathbf{ -3.57 \pm NA }$\\
		\bottomrule[1.2pt]
	\end{adjustbox}
	\vspace{-0.2cm}
	\caption{Averaged predictions, with standard deviations, in terms of test log-likelihood.}
	\label{tab:reg}
	\vspace{-0.8cm}
\end{table}
\vspace{-0.3cm}
\section{Experiments}
We test the proposed WGF framework from two perspectives: $\RN{1})$ the effectiveness of the proposed particle-approximation method for WGF, and $\RN{2})$ the advantages of the WGF framework for policy optimization. For $\RN{1})$, a standard regression model to learn optimal parameter distributions, {\it i.e.}, posterior distributions. For $\RN{2})$, we test our algorithms on several domains in OpenAI $\mathtt{rllab}$ and $\mathtt{Gym}$ \citep{duan2016benchmarking}. 
All experiments are conducted on a single Tesla P100. Detailed settings are given in the SM.

\begin{figure}[ht] \centering
	\begin{tabular}{cc}
		\hspace{-6mm}
		\includegraphics[width=4.4cm]{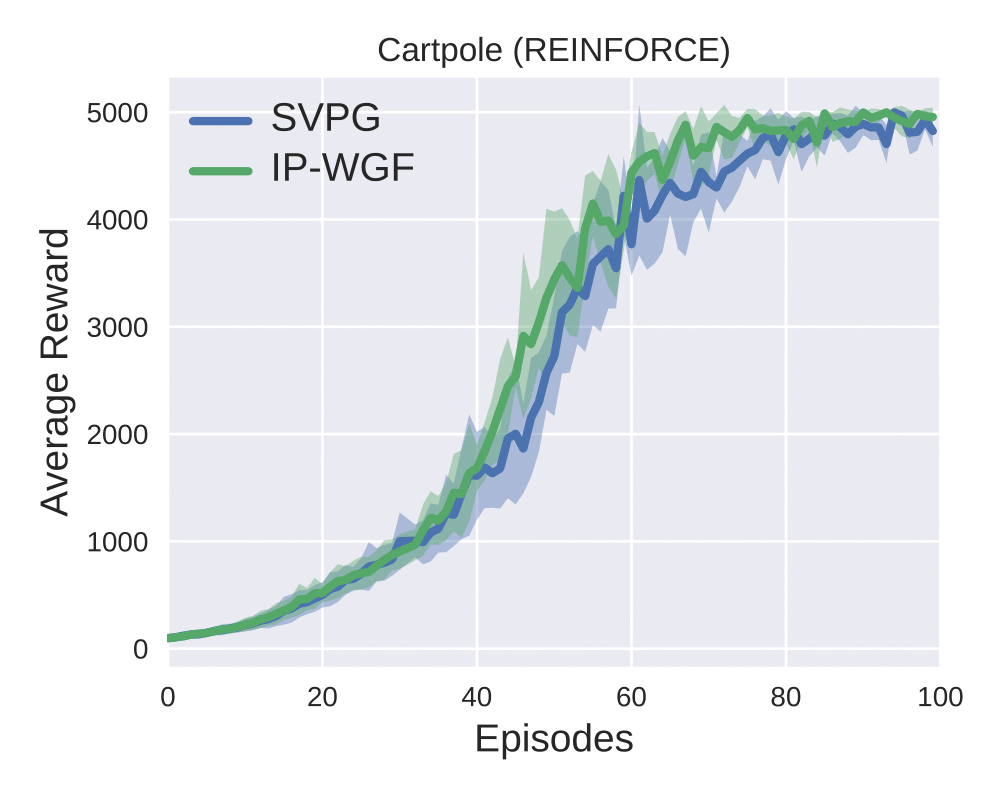}  &
		\hspace{-6mm}
		\includegraphics[width=4.4cm]{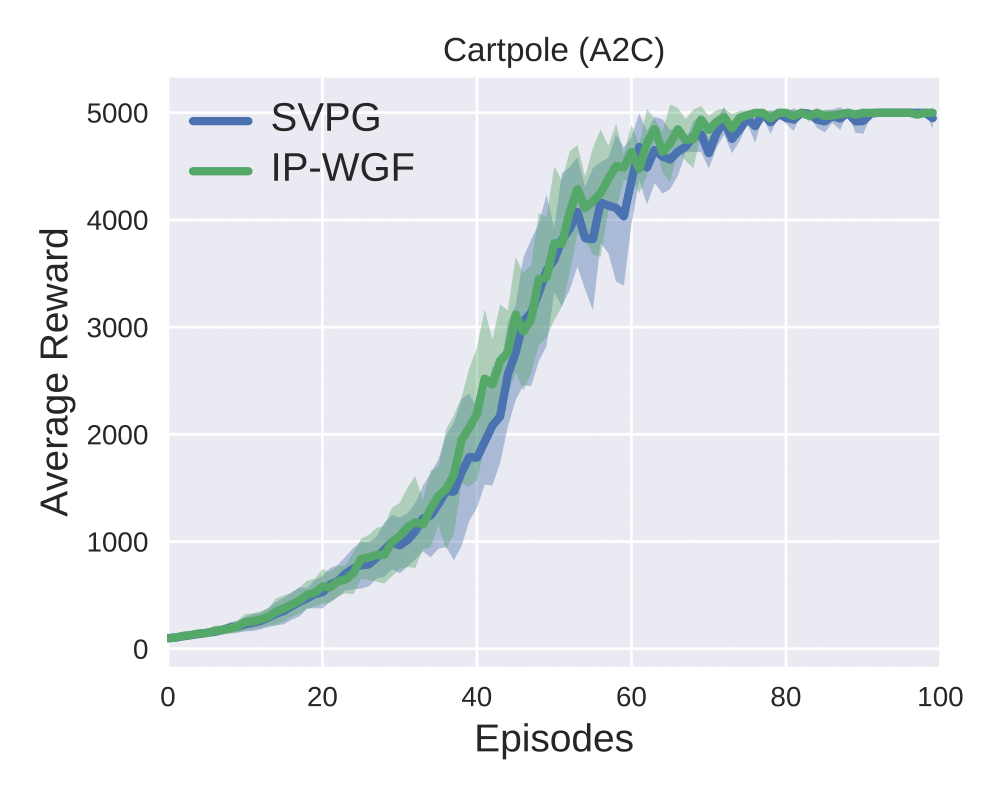}
		\vspace{-3mm}
		\\
		\hspace{-6mm}
		\includegraphics[width=4.4cm]{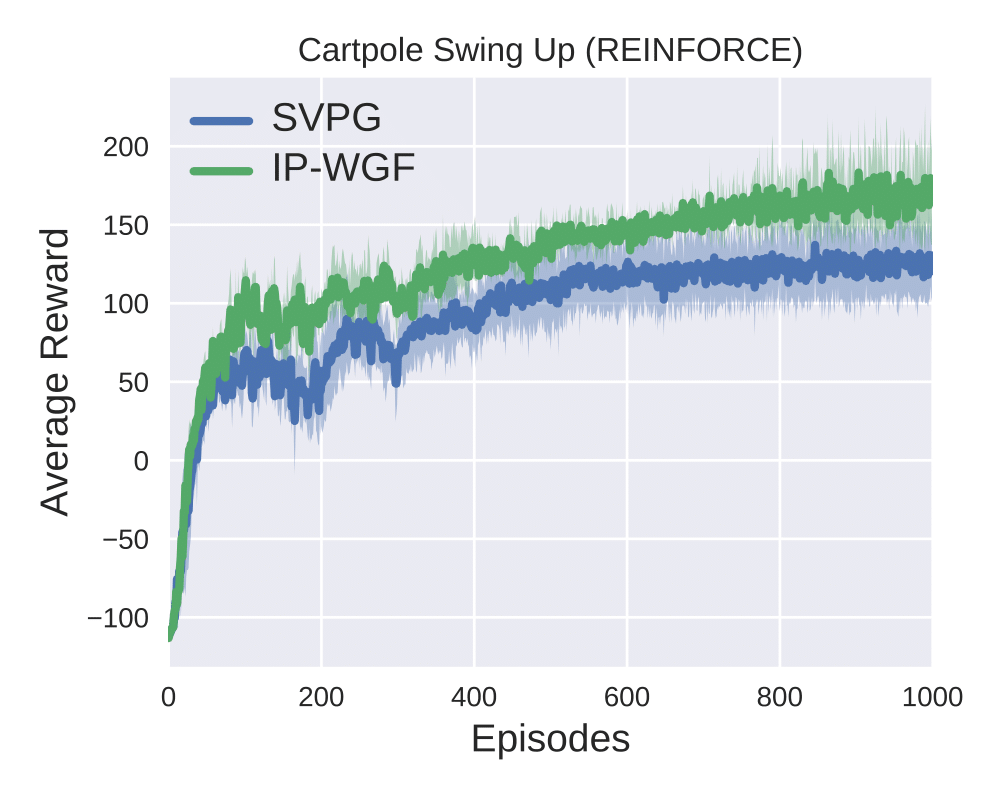}   &
		\hspace{-6mm}
		\includegraphics[width=4.4cm]{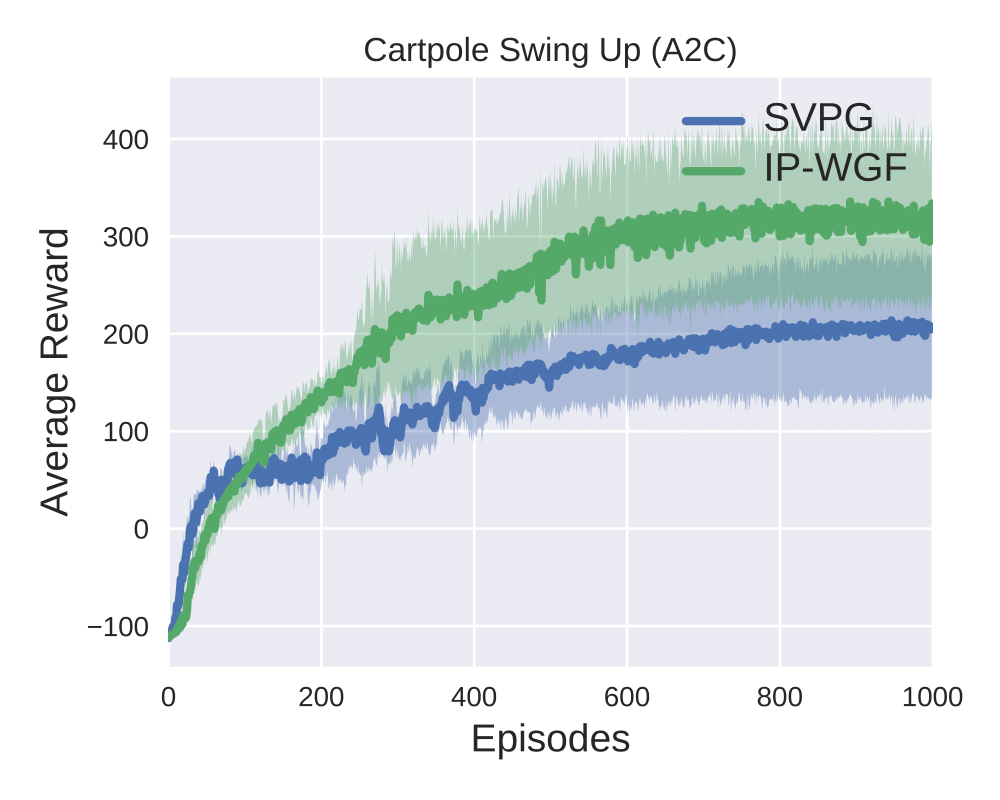}
		\vspace{-3mm}
		\\
		\hspace{-6mm}
		\includegraphics[width=4.4cm]{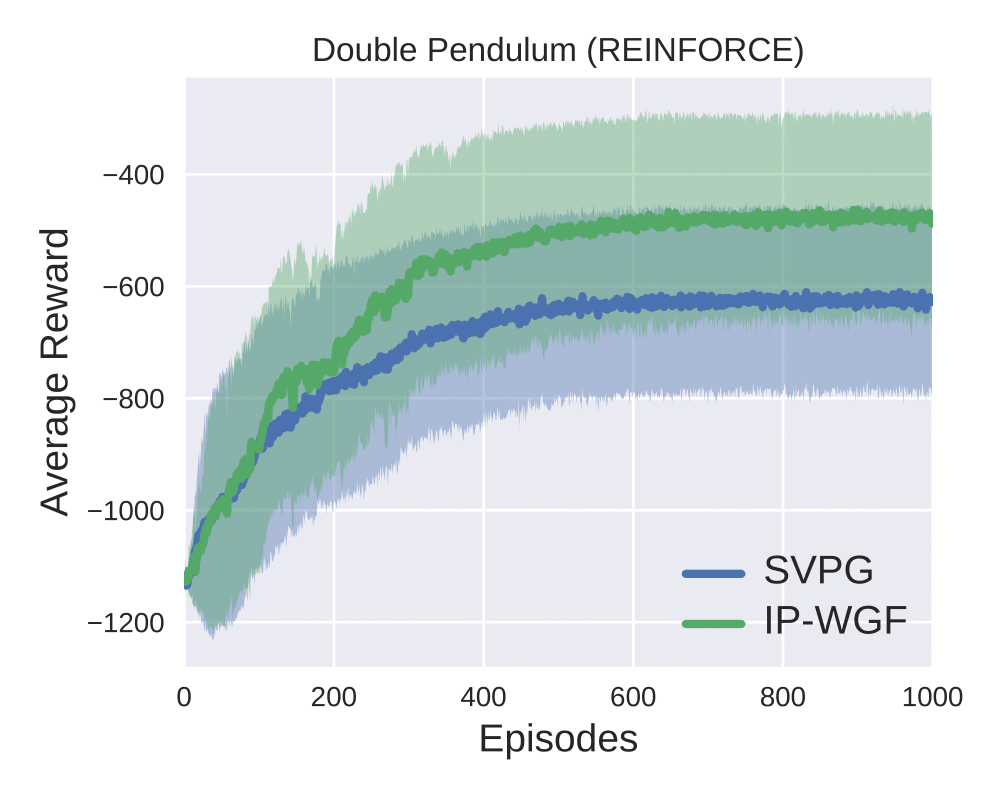} &
		\hspace{-6mm}
		\includegraphics[width=4.4cm]{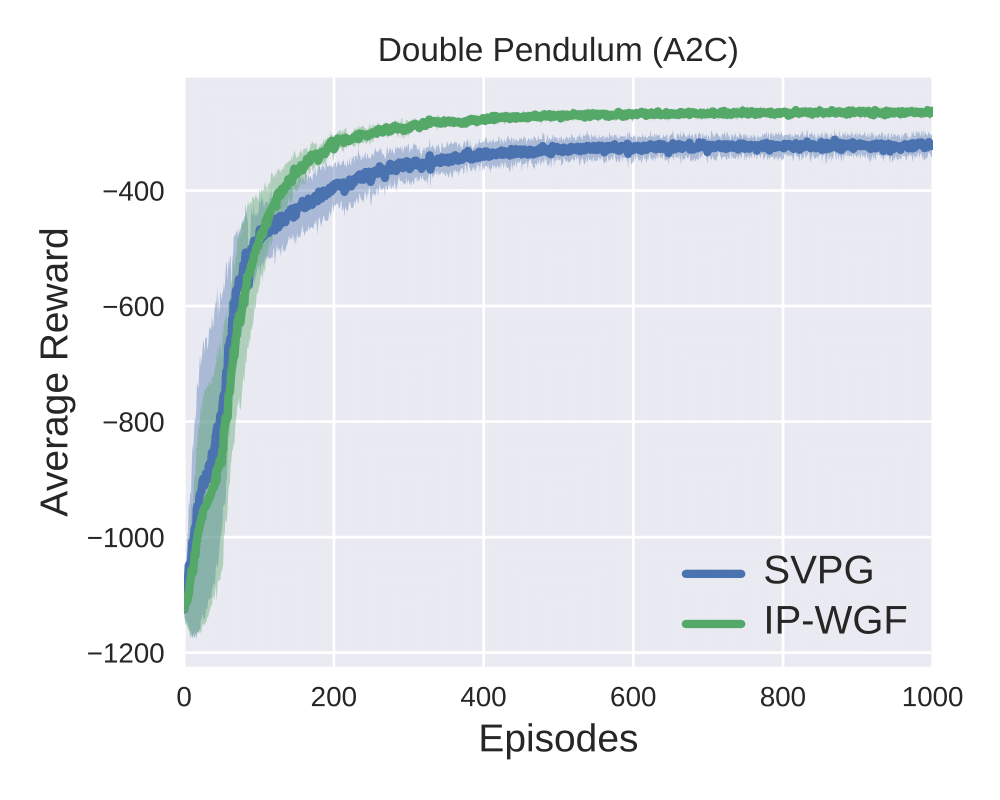}
		\\
	\end{tabular} \vspace{-6mm}
	\caption{{\small Learning curves by IP-WGF and SVPG with REINFORCE and A2C. }}
	\label{fig:rlSVGD0}
	\vspace{-0.5cm}
\end{figure}
\begin{table*}[ht]
	\centering 
	\footnotesize
	\begin{tabular}{c|c| c c| c c | c c | c c}
		\toprule[1.2pt]
		\multicolumn{2}{c}{} & \multicolumn{2}{c}{WGF-DP-V} & \multicolumn{2}{c}{SAC} & \multicolumn{2}{c}{TRPO-GAE} & \multicolumn{2}{c}{DDPG}\\
		\hline
		Domain & Threshold & MaxReturn. & Episodes & MaxReturn & Epsisodes &  MaxReturn & Episodes & MaxReturn & Episodes \\
		\hline
		Swimmer & 100 & \textbf{181.60} & \textbf{76} & 180.83 & 112 &  110.58 & 433  & 49.57 & N/A\\
		Walker & 3000 & \textbf{4978.59} & \textbf{2289} & 4255.05 & 2388 &  3497.81 & 3020 & 2138.42& N/A\\
		Hopper & 2000 & \textbf{3248.76} & \textbf{678} & 3146.51 & 736 & 2604 & 1749 & 1317 & N/A \\
		Humanoid & 2000 & 3077.84 & \textbf{18740} & 2212.51  & 26476 & \textbf{5411.15}& 32261 & 2230.60& 34652 \\
		\bottomrule[1.2pt]
	\end{tabular}
	\vspace{-0.3cm}
	\caption{WGF-DP-V, TRPO, SAC, and DDPG results showing the max average rewards attained and the episodes to cross specific reward thresholds. WGF-DP-V often learn more sample-efficiently than the baselines, and WGF-DP-V can solve difficult domains such as Humanoid better than DDPG. }
	\label{tab:wgf_results}
	\vspace{-0.3cm}
\end{table*}
\begin{figure*}[h!] \centering
	\begin{tabular}{cccc}
		\hspace{-4mm}
		\includegraphics[width=4.4cm]{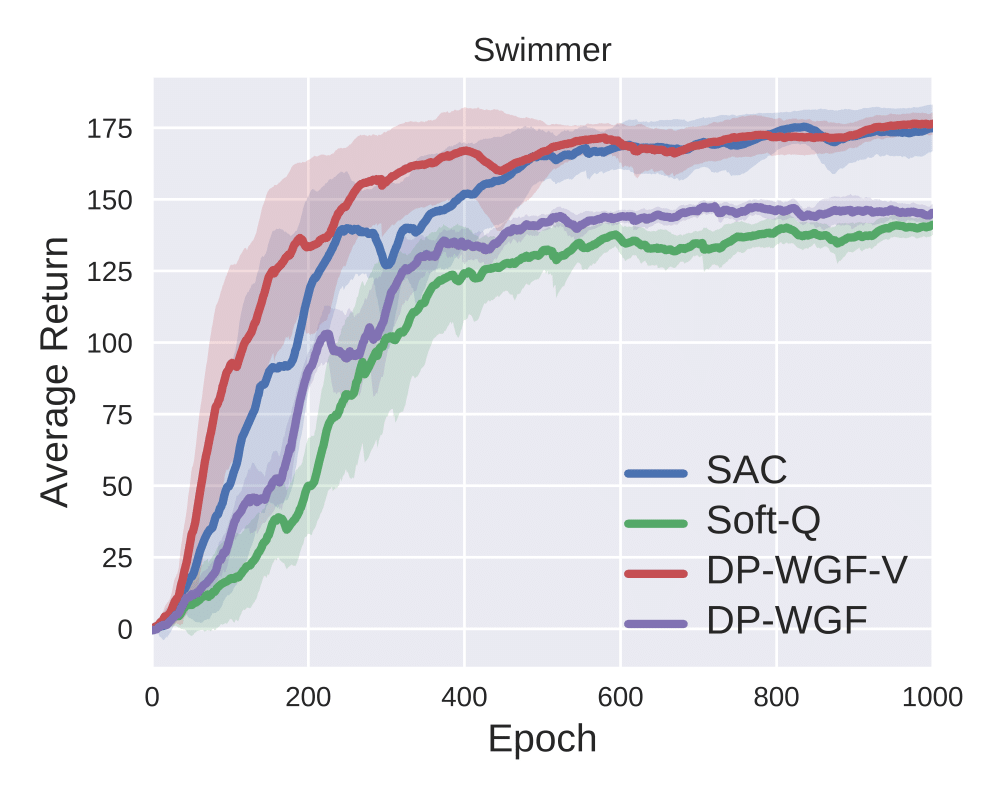}  
		&   \hspace{-6mm}
		\includegraphics[width=4.4cm]{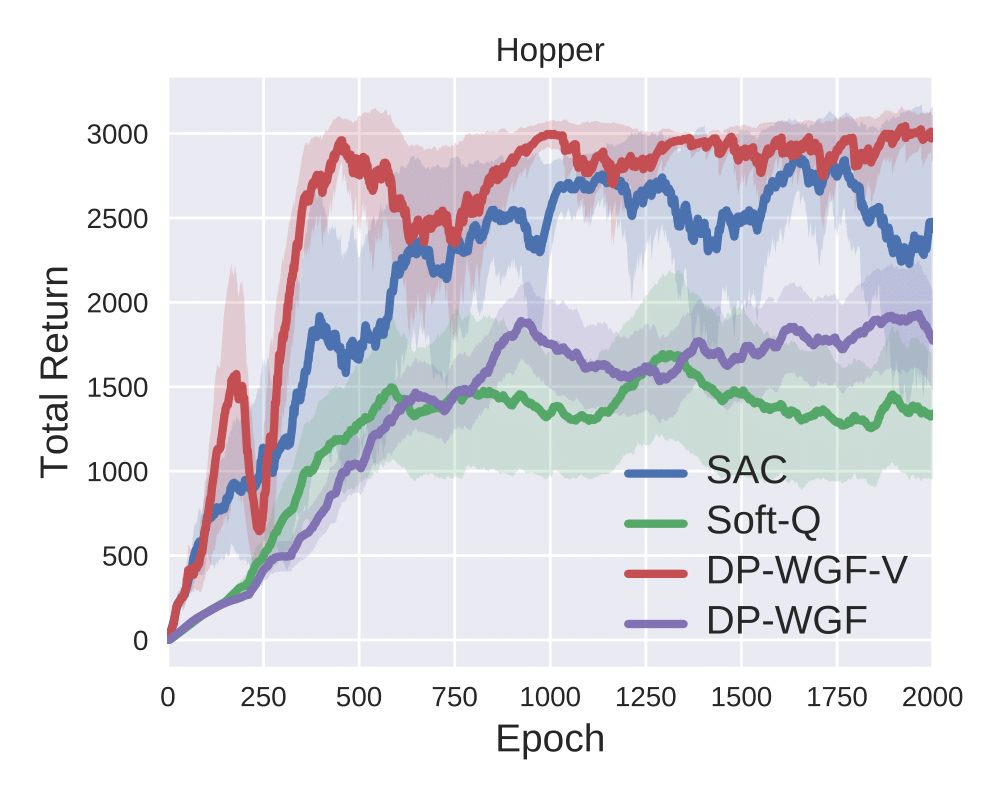} 
		& 		\hspace{-6mm}
		\includegraphics[width=4.4cm]{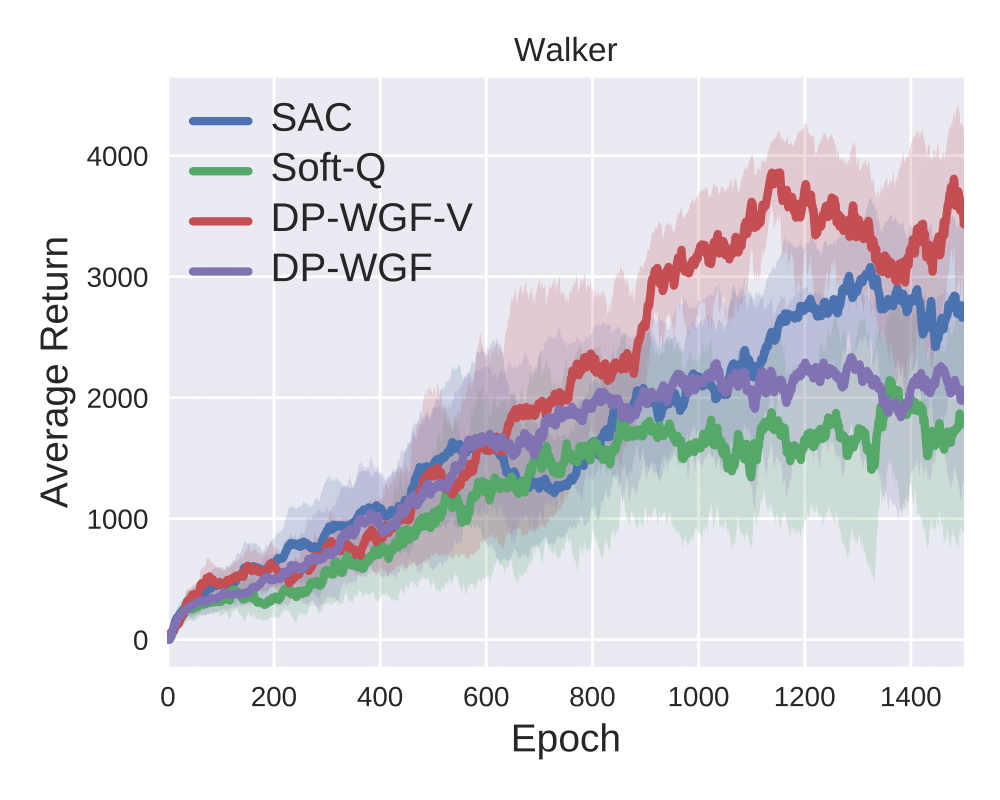}   
		&		\hspace{-6mm}
		\includegraphics[width=4.4cm]{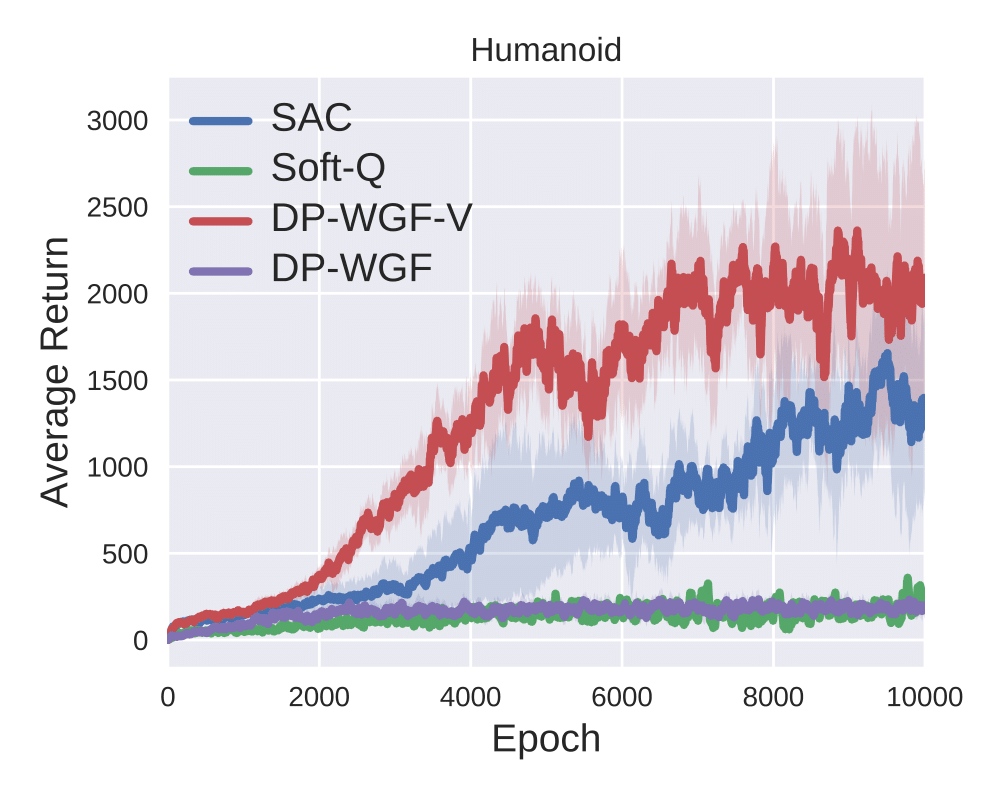}   
		\vspace{-3mm}
		\\
		\hspace{-4mm}
		\includegraphics[width=4.4cm]{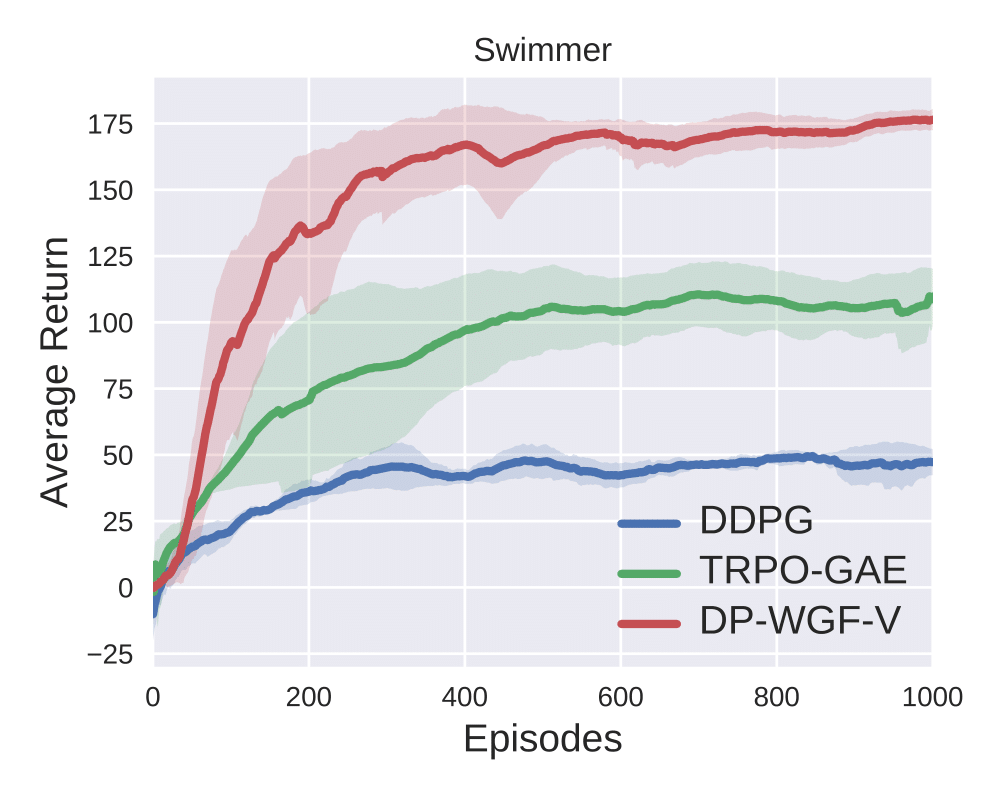}  
		&   \hspace{-6mm}
		\includegraphics[width=4.4cm]{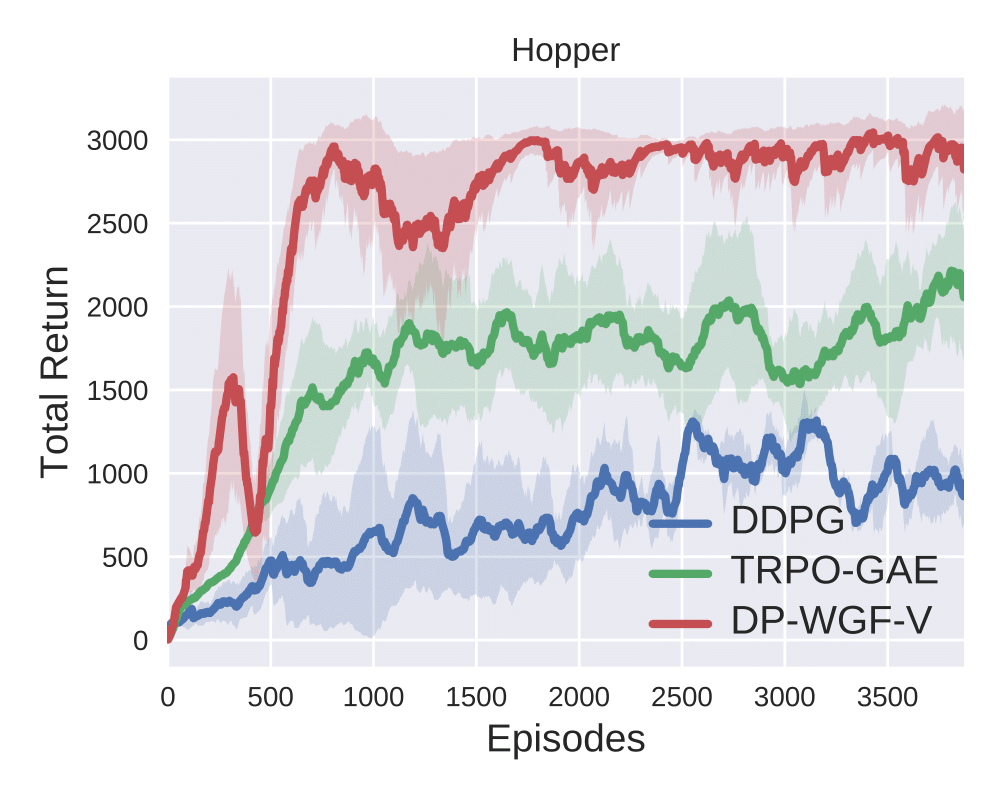} 
		& 		\hspace{-6mm}
		\includegraphics[width=4.4cm]{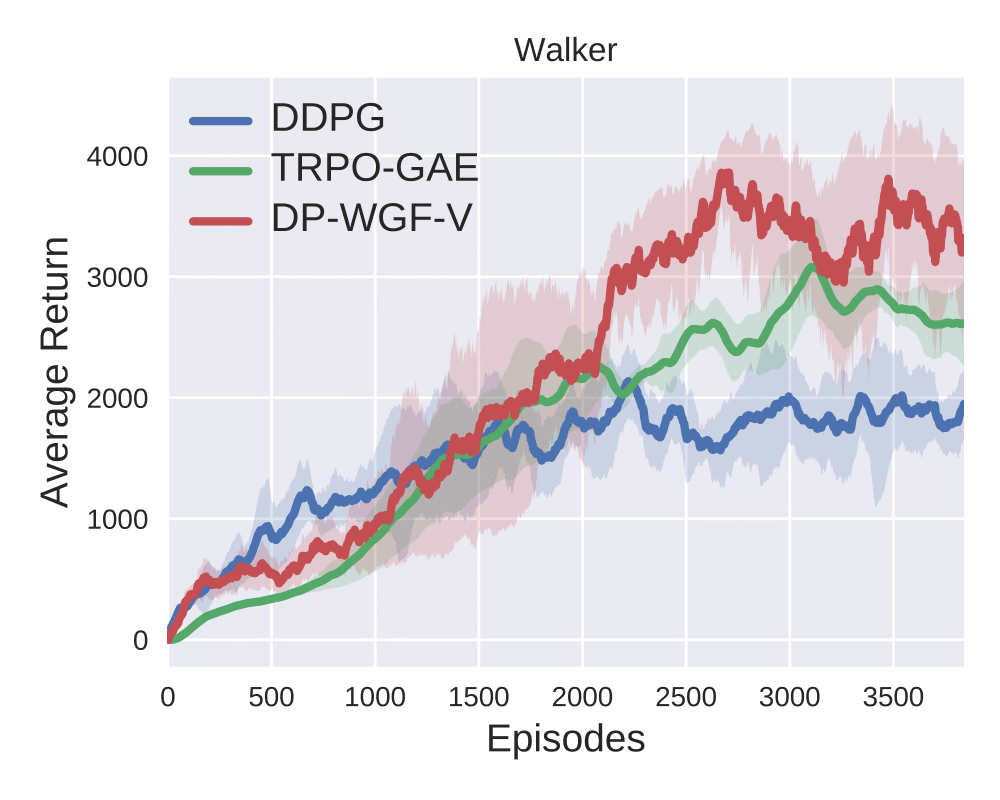}   
		&		\hspace{-6mm}
		\includegraphics[width=4.4cm]{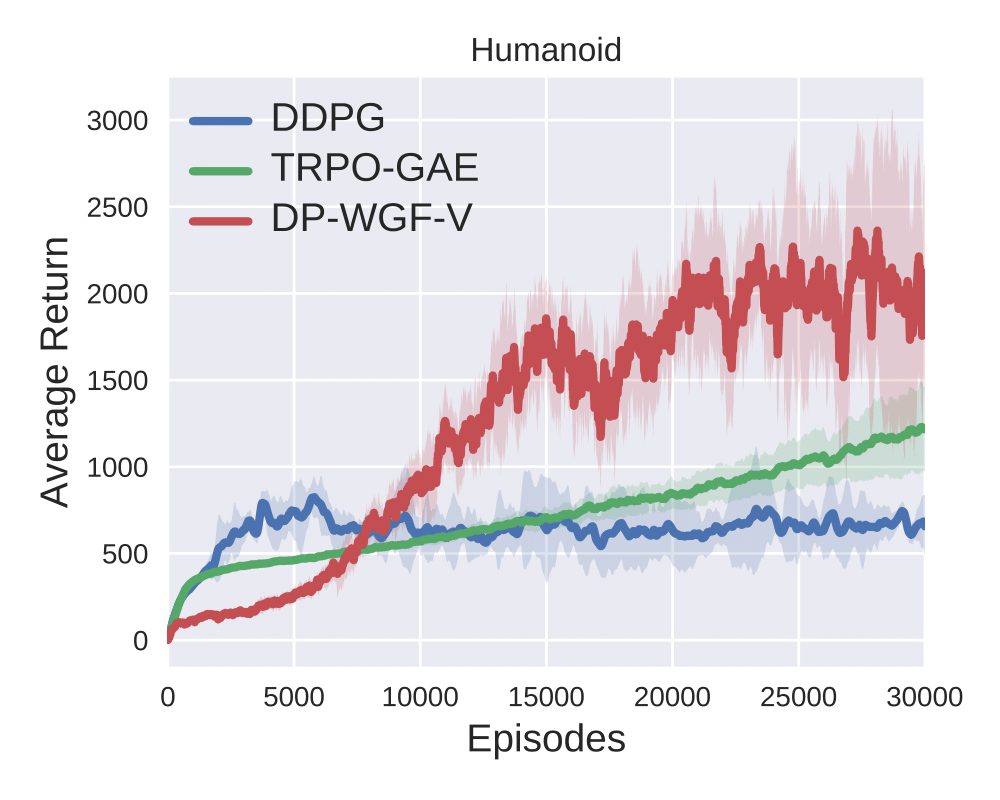} 
	\end{tabular} \vspace{-7mm}
	\caption{{\small Average return in MuJoCo tasks by Soft-Q, SAC and DP-WGF-V (first row), and by DDPG, TRPO-GAE and DP-WGF-V (second row). From left to right, the tasks are: Swimmer, Hopper, Walker and Humanoid, respectively.}}
	\label{fig:rlSVGD1}
	\vspace{-5mm}
\end{figure*}
\vspace{-0.2cm}
\subsection{Regression}
\vspace{-0.2cm}
We use a single-layer BNN as a regression model. The parameters of the BNN are treated probabilistically and optimized with our WGF framework. We compare WGF, SVGD \cite{LiuW:NIPS16}, Bayesian Dropout \cite{GalG:ICML16} and PBP \cite{LobatoA:ICML15}. The RMSprop optimizer is employed. Detailed experimental settings and datasets are described in Section~\ref{supp:exp1} of the SM. We adopt the root-mean-squared error (RMSE) and test log-likelihood as the evaluation criteria. The experimental results are shown in Table~\ref{tab:reg} (complete results are provided in Section~\ref{supp:exp1} of the SM). It is observed that our proposed WGF obtains better results in both metrics, partially due to the flexibility of our particle approximation algorithm, which solves the original WGF problem effectively.

\vspace{-0.2cm}
\subsection{Indirect-policy learning}\vspace{-0.2cm}
For this group of experiments, we compare IP-WGF with SVPG \cite{liu2017stein}, a state-of-the-art method for indirect-policy learning, considering three classical continuous-control tasks: Cartpole Swing-Up, Double Pendulum, and Cartpole. Only policy parameters are updated by IP-WGF or SVPG, while the critics are updated with TD-error. We train our agents for 100 iterations on the easier Cartpole domain and 1000 iterations on the other two domains. 
Following the settings in \citep{liu2017stein, houthooft2016vime, zhang2017learning}, the policy is parameterized as a two-layer (25-16 hidden units) neural network with $\mathtt{tanh}$ activation function. The maximum horizon length is set to 500. A sample size of 5000 is used for policy gradient estimation. We use $M=16$ particles to approximate parameter distributions, and $h=0.1$ as the discretized stepsize. 

REINFORCE \citep{williams1992simple} and advantage actor critic~(A2C) \citep{schulman2015high} are used as strategies of policy learning.  Figure~\ref{fig:rlSVGD1} plots the mean (dark curves) and standard derivation (light areas) of rewards over 5 runs. It is clear that in all tasks IP-WGF consistently converges faster than SVPG, and finally converges to higher average rewards. The results are comparable to \citep{houthooft2016vime}.
The experiments demonstrate that employing the Wasserstein gradient flows on policy optimization improves the performance, as suggested by our theory.
%
\subsection{Direct-policy Learning}\vspace{-0.2cm}
We compare our DP-WGF and DP-WGF-V frameworks with existing off-policy and on-policy deep RL algorithms on several tasks in MuJoCo, {\it e.g.}, SAC, Soft-Q, DDPG (off-policy) and TRPO-GAE (on-policy). Our DP-WGF-V is considered to be an off-policy actor-critic method. For all methods, value function and policy are parameterized as two-layer (128-128 hidden units) neural networks with $\mathtt{tanh}$ as the activation function. The maximum horizon length is set to 1000 when simulating expected total rewards. Three easier tasks (Swimmer, Hopper and Walker) in MuJoCo can be solved by a wide range of algorithms; while the more complex benchmark, the 21-dimensional Humanoid, is known to be very difficult to solve with off-policy algorithms \citep{duan2016benchmarking}. Implementation details of the algorithms are specified in Section~\ref{supp:DPL} of the SM.

\vspace{-4mm}
\paragraph{Effectiveness of the Wasserstein trust-region}
We evaluate DP-WGF-V against SAC, and DP-WGF against Soft-Q on four Mujoco tasks, as they are closely related to our algorithms. Figure~\ref{fig:rlSVGD1} (first row) plots average returns over epochs on the tasks. Similarly, our WGF-based methods converge faster and better than their counterparts due to the introduction of WGFs. Furthermore, by variance reduction, DP-WGF-V significantly outperforms DP-WGF on all tasks.
\vspace{-0.8cm}
\paragraph{Comparisons with popular baselines}
Finally we compare DP-WGF-V with TRPO-GAE \cite{schulman2015high} and DDPG \cite{LillicrapHPHETSW:ICLR16} on the same Mujoco tasks. In general, TRPO-GAE has been a state-of-the-art method for policy optimization. Figure~\ref{fig:rlSVGD1} (second row) plots average returns over episodes, and it is observed that DP-WGF-V consistently outperforms other algorithms. 
Table~\ref{tab:wgf_results} summarizes some key statistics, including the best attained average rewards and the episodes to reach the reward thresholds. It is observed that DP-WGF-V consistently outperform TRPO-GAE and DDPG in terms of sample complexity, and often achieves higher rewards than TRPO-GAE. A particularly notable case, on Humanoid, shows DP-WGF-V substantially outperforms TRPO-GAE in terms of sample efficiency, while DDPG cannot learn a good policy at all.

	\vspace{-3mm}
	\section{Conclusion}
	\vspace{-0.6mm}
	We lift policy optimization to the space of probabilistic distributions, and interpret it as Wasserstein gradient flows. Two types of WGFs are defined for the task, one on the parameter-distribution space and the other on the policy-distribution space. The WGFs are solved by a new particle-approximation-based algorithm, where gradients of particles are calculated in closed forms. Under some circumstance, optimization on probability-distribution space is convex, thus it is easier to deal with compared to existing methods\footnote{In parallel to our work, \cite{richemond2017wasserstein} explores Wasserstein gradient flows for diffusing policy.}. Experiments are conducted on a number of reinforcement-learning tasks, demonstrating the superiority of the proposed framework compared to related algorithms.
	

	\paragraph{Acknowledgements}
	We acknowledge Tuomas Haarnoja et al. for making their code public and thank Ronald Parr for insightful advice. This research was supported in part by DARPA, DOE, NIH, ONR and NSF.
	\nocite{Tao2018glbo}
	
	\bibliography{reference}

\begin{thebibliography}{49}
\providecommand{\natexlab}[1]{#1}
\providecommand{\url}[1]{\texttt{#1}}
\expandafter\ifx\csname urlstyle\endcsname\relax
  \providecommand{\doi}[1]{doi: #1}\else
  \providecommand{\doi}{doi: \begingroup \urlstyle{rm}\Url}\fi

\bibitem[Ambrosio et~al.(2005)Ambrosio, Gigli, and Savar\'{e}]{Ambrosio:book05}
Ambrosio, L., Gigli, N., and Savar\'{e}, G.
\newblock \emph{Gradient Flows in Metric Spaces and in the Space of Probability
  Measures}.
\newblock Lectures in Mathematics ETH Zürich, 2005.

\bibitem[Arjovsky et~al.()Arjovsky, Chintala, and Bottou]{ArjovskyCB:arxiv17}
Arjovsky, M., Chintala, S., and Bottou, L.
\newblock Wasserstein {GAN}.
\newblock In \emph{NIPS}.

\bibitem[Blundell et~al.(2015)Blundell, Cornebise, Kavukcuoglu, and
  Wierstra]{blundell2015weight}
Blundell, C., Cornebise, J., Kavukcuoglu, K., and Wierstra, D.
\newblock Weight uncertainty in neural networks.
\newblock In \emph{ICML}, 2015.

\bibitem[Brockman et~al.(2016)Brockman, Cheung, Pettersson, Schneider,
  Schulman, Tang, and Zaremba]{openaigym}
Brockman, G., Cheung, V., Pettersson, L., Schneider, J., Schulman, J., Tang,
  J., and Zaremba, W.
\newblock Openai gym, 2016.

\bibitem[Carrillo et~al.(2017)Carrillo, Craig, and
  Patacchini]{CarrilloCP:arxiv17}
Carrillo, J.~A., Craig, K., and Patacchini, F.~S.
\newblock A blob method for diffusion.
\newblock \penalty0 (arXiv:1709.09195), 2017.

\bibitem[Chen et~al.(2018)Chen, Zhang, Wang, Li, and Chen]{ChenZWLC:tech18}
Chen, C., Zhang, R., Wang, W., Li, B., and Chen, L.
\newblock A unified particle-optimization framework for scalable {B}ayesian
  sampling.
\newblock In \emph{UAI}, 2018.

\bibitem[Cottet \& Koumoutsakos(2000)Cottet and Koumoutsakos]{CottetK:00}
Cottet, G.~H. and Koumoutsakos, P.~D.
\newblock \emph{Vortex methods}.
\newblock Cambridge University Press, 2000.

\bibitem[Duan et~al.(2016)Duan, Chen, Houthooft, Schulman, and
  Abbeel]{duan2016benchmarking}
Duan, Y., Chen, X., Houthooft, R., Schulman, J., and Abbeel, P.
\newblock Benchmarking deep reinforcement learning for continuous control.
\newblock In \emph{ICML}, 2016.

\bibitem[Fortunato et~al.(2018)Fortunato, Azar, Piot, Menick, Osband, Graves,
  Mnih, Munos, Hassabis, Pietquin, Blundell, and Legg]{fortunato2017noisy}
Fortunato, M., Azar, M.~G., Piot, B., Menick, J., Osband, I., Graves, A., Mnih,
  V., Munos, R., Hassabis, D., Pietquin, O., Blundell, C., and Legg, S.
\newblock Noisy networks for exploration.
\newblock In \emph{ICLR}, 2018.

\bibitem[Gal \& Ghahramani(2016)Gal and Ghahramani]{GalG:ICML16}
Gal, Y. and Ghahramani, Z.
\newblock Dropout as a {B}ayesian approximation: representing model uncertainty
  in deep learning.
\newblock In \emph{ICML}, 2016.

\bibitem[Gobbino(1999)]{Gobbino:AMPA99}
Gobbino, M.
\newblock Minimizing movements and evolution problems in {E}uclidean spaces.
\newblock \emph{Annali di Matematica Pura ed Applicata}, 1999.

\bibitem[Gu et~al.(2017)Gu, Lillicrap, Ghahramani, Turner, and Levine]{gu2016q}
Gu, S., Lillicrap, T., Ghahramani, Z., Turner, R.~E., and Levine, S.
\newblock Q-prop: Sample-efficient policy gradient with an off-policy critic.
\newblock In \emph{ICLR}, 2017.

\bibitem[Haarnoja et~al.(2017)Haarnoja, Tang, Abbeel, and
  Levine]{HaarnojaTAL:ICML17}
Haarnoja, T., Tang, H., Abbeel, P., and Levine, S.
\newblock Reinforcement learning with deep energy-based policies.
\newblock In \emph{ICML}, 2017.

\bibitem[Haarnoja et~al.(2018)Haarnoja, Zhou, Abbeel, and
  Levine]{HaarnojaZAL:arxiv18}
Haarnoja, T., Zhou, A., Abbeel, P., and Levine, S.
\newblock Soft actor-critic: Off-policy maximum entropy deep reinforcement
  learning with a stochastic actor.
\newblock In \emph{ICML}, 2018.

\bibitem[Henderson et~al.(2018)Henderson, Islam, Bachman, Pineau, Precup, and
  D.~Meger]{henderson2017deep}
Henderson, P., Islam, R., Bachman, P., Pineau, J., Precup, D., and D.~Meger, D.
\newblock Deep reinforcement learning that matters.
\newblock In \emph{AAAI}, 2018.

\bibitem[Hern\'{a}ndez-Lobato \& Adams(2015)Hern\'{a}ndez-Lobato and
  Adams]{LobatoA:ICML15}
Hern\'{a}ndez-Lobato, J.~M. and Adams, R.~P.
\newblock Probabilistic backpropagation for scalable learning of {B}ayesian
  neural networks.
\newblock In \emph{ICML}, 2015.

\bibitem[Hinton et~al.(2012)Hinton, Srivastava, and Swersky]{hinton2012rmsprop}
Hinton, G.~E., Srivastava, N., and Swersky, K.
\newblock Rmsprop: Divide the gradient by a running average of its recent
  magnitude.
\newblock \emph{Neural Networks for Machine Learning, Coursera}, 2012.

\bibitem[Houthooft et~al.(2016)Houthooft, Chen, Duan, Schulman, Turck, and
  Abbeel]{houthooft2016vime}
Houthooft, R., Chen, X., Duan, Y., Schulman, J., Turck, F.~D., and Abbeel, P.
\newblock {VIME}: Variational information maximizing exploration.
\newblock In \emph{NIPS}, 2016.

\bibitem[Jordan et~al.(1998)Jordan, Kinderlehrer, and Otto]{JordanKO:MA98}
Jordan, R., Kinderlehrer, D., and Otto, F.
\newblock The variational formulation of the {F}okker-{P}lanck equation.
\newblock \emph{SIAM Journal on Mathematical Analysis}, 29\penalty0
  (1):\penalty0 1--17, 1998.

\bibitem[Kaelbling et~al.(1996)Kaelbling, Littman, and
  Moore]{kaelbling1996reinforcement}
Kaelbling, L.~P., Littman, M.~L., and Moore, A.~W.
\newblock Reinforcement learning: A survey.
\newblock In \emph{JAIR}, 1996.

\bibitem[Kakade(2002)]{kakade2002natural}
Kakade, S.~M.
\newblock A natural policy gradient.
\newblock In \emph{NIPS}, 2002.

\bibitem[Kingma \& Ba(2015)Kingma and Ba]{kingma2014adam}
Kingma, D. and Ba, J.
\newblock Adam: A method for stochastic optimization.
\newblock In \emph{ICLR}, 2015.

\bibitem[Li(2017)]{Li:arxiv17}
Li, Y.
\newblock Deep reinforcement learning: An overview.
\newblock \penalty0 (arXiv:1701.07274), 2017.

\bibitem[Lillicrap et~al.(2016)Lillicrap, Hunt, Pritzel, Heess, Erez, Tassa,
  Silver, and Wierstra]{LillicrapHPHETSW:ICLR16}
Lillicrap, T.~P., Hunt, J.~J., Pritzel, A., Heess, N., Erez, T., Tassa, Y.,
  Silver, D., and Wierstra, D.
\newblock Continuous control with deep reinforcement learning.
\newblock In \emph{ICLR}, 2016.

\bibitem[Liu \& Wang(2016)Liu and Wang]{LiuW:NIPS16}
Liu, Q. and Wang, D.
\newblock Stein variational gradient descent: {A} general purpose {Bayesian}
  inference algorithm.
\newblock In \emph{NIPS}, 2016.

\bibitem[Liu et~al.(2017)Liu, Ramachandran, Liu, and Peng]{liu2017stein}
Liu, Y., Ramachandran, P., Liu, Q., and Peng, J.
\newblock Stein variational policy gradient.
\newblock In \emph{UAI}, 2017.

\bibitem[Mnih et~al.(2013)Mnih, Kavukcuoglu, Silver, Graves, Antonoglou,
  Wierstra, and Riedmiller]{mnih2013playing}
Mnih, V., Kavukcuoglu, K., Silver, D., Graves, A., Antonoglou, I., Wierstra,
  D., and Riedmiller, M.
\newblock Playing {A}tari with deep reinforcement learning.
\newblock \emph{arXiv:1312.5602}, 2013.

\bibitem[Mnih et~al.(2015)Mnih, Kavukcuoglu, Silver, Rusu, Veness, Bellemare,
  Graves, Riedmiller, Fidjeland, Ostrovski, Petersen, Beattie, Sadik,
  Antonoglou, King, Kumaran, Wierstra, Legg, and Hassabis]{mnih2015human}
Mnih, V., Kavukcuoglu, K., Silver, D., Rusu, A., Veness, J., Bellemare, M.,
  Graves, A., Riedmiller, M., Fidjeland, A., Ostrovski, G., Petersen, S.,
  Beattie, C., Sadik, A., Antonoglou, I., King, H., Kumaran, D., Wierstra, D.,
  Legg, S., and Hassabis, D.
\newblock Human-level control through deep reinforcement learning.
\newblock \emph{Nature}, 2015.

\bibitem[Mnih et~al.(2016)Mnih, Badia, Mirza, Graves, Lillicrap, Harley,
  Silver, and Kavukcuoglu]{mnih2016asynchronous}
Mnih, V., Badia, A.~P., Mirza, M., Graves, A., Lillicrap, T., Harley, T.,
  Silver, D., and Kavukcuoglu, K.
\newblock Asynchronous methods for deep reinforcement learning.
\newblock In \emph{ICML}, pp.\  1928--1937, 2016.

\bibitem[Nachum et~al.(2017)Nachum, Norouzi, Xu, and
  Schuurmans]{NachumNXS:arxiv17}
Nachum, O., Norouzi, M., Xu, K., and Schuurmans, D.
\newblock Trust-{PCL}: An off-policy trust region method for continuous
  control.
\newblock \penalty0 (arXiv:1707.01891), 2017.

\bibitem[O'Donoghue et~al.(2016)O'Donoghue, Munos, Kavukcuoglu, and
  Mnih]{o2016pgq}
O'Donoghue, B., Munos, R., Kavukcuoglu, K., and Mnih, V.
\newblock Pgq: Combining policy gradient and q-learning.
\newblock \emph{arXiv:1611.01626}, 2016.

\bibitem[Plappert et~al.(2018)Plappert, Houthooft, Dhariwal, Sidor, Chen, Chen,
  Asfour, Abbeel, and Andrychowicz]{plappert2017parameter}
Plappert, M., Houthooft, R., Dhariwal, P., Sidor, S., Chen, R.~Y., Chen, X.,
  Asfour, T., Abbeel, P., and Andrychowicz, M.
\newblock Parameter space noise for exploration.
\newblock In \emph{ICLR}, 2018.

\bibitem[Richemond \& Maginnis(2017)Richemond and
  Maginnis]{richemond2017wasserstein}
Richemond, P.~H. and Maginnis, B.
\newblock On wasserstein reinforcement learning and the fokker-planck equation.
\newblock \emph{arXiv preprint arXiv:1712.07185}, 2017.

\bibitem[Risken(1989)]{Risken:FPE89}
Risken, H.
\newblock \emph{The {F}okker-{P}lanck equation}.
\newblock Springer-Verlag, New York, 1989.

\bibitem[Rulla(1996)]{Rulla:NA96}
Rulla, J.
\newblock Error analysis for implicit approximations to solutions to {C}auchy
  problems.
\newblock \emph{SIAM Journal on Numerical Analysis}, 33\penalty0 (1):\penalty0
  68--87, 1996.

\bibitem[Russo(1990)]{Russo:CPAM90}
Russo, G.
\newblock Deterministic diffusion of particles.
\newblock \emph{Communications on Pure and Applied Mathematics}, 1990.

\bibitem[Schulman et~al.(2015)Schulman, Levine, Abbeel, Jordan, and
  Moritz]{schulman2015trust}
Schulman, J., Levine, S., Abbeel, P., Jordan, M., and Moritz, P.
\newblock Trust region policy optimization.
\newblock In \emph{ICML}, 2015.

\bibitem[Schulman et~al.(2016)Schulman, Moritz, Levine, Jordan, and
  Abbeel]{schulman2015high}
Schulman, J., Moritz, P., Levine, S., Jordan, M., and Abbeel, P.
\newblock High-dimensional continuous control using generalized advantage
  estimation.
\newblock In \emph{ICLR}, 2016.

\bibitem[Schulman et~al.(2017{\natexlab{a}})Schulman, Abbeel, and
  Chen]{schulman2017equivalence}
Schulman, J., Abbeel, P., and Chen, X.
\newblock Equivalence between policy gradients and soft $q$-learning.
\newblock \penalty0 (arXiv:1704.06440), 2017{\natexlab{a}}.

\bibitem[Schulman et~al.(2017{\natexlab{b}})Schulman, Wolski, Dhariwal,
  Radford, and Klimov]{SchulmanWDRK:arxiv17}
Schulman, J., Wolski, F., Dhariwal, P., Radford, A., and Klimov, O.
\newblock Proximal policy optimization algorithms.
\newblock \penalty0 (arXiv:1707.06347), 2017{\natexlab{b}}.

\bibitem[Silver et~al.(2014)Silver, Lever, Heess, Degris, Wierstra, and
  Riedmiller]{silver2014deterministic}
Silver, D., Lever, G., Heess, N., Degris, T., Wierstra, D., and Riedmiller, M.
\newblock Deterministic policy gradient algorithms.
\newblock In \emph{ICML}, 2014.

\bibitem[Sutton \& Barto(1998)Sutton and Barto]{sutton1998reinforcement}
Sutton, R.~S. and Barto, A.~G.
\newblock \emph{Reinforcement learning: An introduction}.
\newblock 1998.

\bibitem[Sutton et~al.(2000)Sutton, McAllester, Singh, and
  Mansour]{SuttonMSM:NIPS00}
Sutton, R.~S., McAllester, D., Singh, S., and Mansour, Y.
\newblock Policy gradient methods for reinforcement learning with function
  approximation.
\newblock In \emph{NIPS}, 2000.

\bibitem[Tao et~al.(2018)Tao, Chen, Zhang, Henao, and Carin]{Tao2018glbo}
Tao, C., Chen, L., Zhang, R., Henao, R., and Carin, L.
\newblock Variational inference and model selection with generalized evidence
  bounds.
\newblock In \emph{ICML}, 2018.

\bibitem[Villani(2008)]{Villani:08}
Villani, C.
\newblock \emph{Optimal transport: old and new}.
\newblock Springer Science \& Business Media, 2008.

\bibitem[Watkins \& Dayan(1992)Watkins and Dayan]{WatkinsD:MLJ92}
Watkins, C.~J. and Dayan, P.
\newblock {$Q$}-learning.
\newblock \emph{Machine Learning}, 1992.

\bibitem[Welling \& Teh(2011)Welling and Teh]{WellingT:ICML11}
Welling, M. and Teh, Y.~W.
\newblock Bayesian learning via stochastic gradient {L}angevin dynamics.
\newblock In \emph{ICML}, 2011.

\bibitem[Williams(1992)]{williams1992simple}
Williams, R.~J.
\newblock Simple statistical gradient-following algorithms for connectionist
  reinforcement learning.
\newblock \emph{Machine Learning}, 1992.

\bibitem[Zhang et~al.(2018)Zhang, Li, Chen, and Carin]{zhang2017learning}
Zhang, R., Li, C., Chen, C., and Carin, L.
\newblock Learning structural weight uncertainty for sequential
  decision-making.
\newblock In \emph{AISTATS}, 2018.

\end{thebibliography}
	\bibliographystyle{icml2018}

	%
	%
	%
	\appendix
	
	\twocolumn[
	\icmltitle{\Large Supplemental Material for \\ Policy Optimization as Wasserstein Gradient Flows}
	]
	\section{More Details on Preliminaries}
	We provide more details on some parts of the preliminaries section in the main text.
	\subsection{Gradient Flows on the Euclidean Space}\label{sec:gf_euc1}
	For a smooth (convex) function\footnote{We will focus on the convex case since this the case for many gradient flows on the space of probability measures, detailed later.} $F: \mathbb{R}^n \rightarrow \mathbb{R}$, a starting point $\xb_0 \in \mathbb{R}^n$. The gradient flow of $F(\xb)$ is defined as the solution of the following function:
	\begin{align}\label{eq:gf_euc1}
	\left\{
	\begin{array}{ll}
	\frac{\mathrm{d}\xb}{\mathrm{d}t} = -\nabla F(\xb(t)), & \text{for }t > 0\\
	\xb(0) = \xb_0
	\end{array}
	\right.
	\end{align}
	This is a standard Cauchy problem \cite{Rulla:NA96}, which has a unique solution if $\nabla F$ is Lipschitz continuous. When $F$ is non-differentiable, we can replace the gradient with the subgradient, defined as
	$\partial F(\xb) \triangleq \{\pp \in \mathbb{R}^n: F(\yb) \geq F(\xb) + \pb \cdot (\yb - \xb), \forall \yb \in \mathbb{R}^n\}$. Note $\partial F(\xb) = \{\nabla F(\xb)\}$ if $F$ is differentiable at $\xb$. In this case, the gradient flow formula above is replaced as: $\frac{\mathrm{d}\xb}{\mathrm{d}t} \in -\partial F(\xb(t))$.
	
	\paragraph{Numerical solution}
	Exact solution to the gradient-flow problem \eqref{eq:gf_euc1} is typically intractable. Numerical methods is a default choice. A standard method to solve \eqref{eq:gf_euc1} is called {\em Minimizing Movement Scheme} (MMS), which is an iterative scheme that evolves $\xb$ along the gradient of $F$ on the current point for a small step in each iteration. Specifically, let the current point to be $\xb_k$, the next point is defined as $\xb_{k+1} = \xb_k - \nabla F(\xb_{k+1}) h$, with $h$ being the stepsize. Note $\xb_{k+1}$ is equivalent to solving the following optimization problem:
	\begin{align*}
	\xb_{k+1} = \arg\min_{\xb}F(\xb) + \frac{\|\xb - \xb_k\|^2}{2h}~.
	\end{align*}
	
	To explicitly spell out the dependency of $\xb_k$ w.r.t.\! $h$, we rewrite $\xb_k$ as $\xb_k^{(h)}$. The numerical scheme can be proved to be accurate. Specifically, define $\vb_{k+1}^{(h)} \triangleq \frac{\xb_{k+1}^{(h)} - \xb_{k}^{(h)}}{h}$. Also define two curves $\xb^h, \tilde{\xb}^h: [0, T] \rightarrow \mathbb{R}^n$ for  $t \in (kh, (k+1)h]$ as:
	\begin{align*}
	\xb^h(t) = \xb_{k+1}^{(h)},~~ \tilde{\xb}^{h}(t) = \xb_k^{(h)} + (t - kh)\vb_{k+1}^{(h)}~.	
	\end{align*}
	The MMS is proved to converge to the original gradient flow \cite{Ambrosio:book05}, stated in the following theorem.
	\begin{theorem}
		Let $\vb^h(t) \triangleq \vb_{k+1}^{(h)}$ defined above. Suppose $F(\xb_0) < +\infty$ and $\inf F > -\infty$. If\footnote{$h$ can also be a decreasing-stepsize sequence $\{h_k\}$ such that $h_k \rightarrow 0$.} $h \rightarrow 0$, $\tilde{\xb}^h$ and $\xb^h$ converge uniformly to a same curve $\xb(t)$, and $\vb^h$ weakly converges in $\mathcal{L}^2$ to a vector function $\vb(t)$, such that $\frac{\mathrm{d}\xb}{\mathrm{d}t} = \vb$. Furthermore,
		if the partial derivatives of $F$ exist and are continuous, we have $\vb(t) = -\nabla F(\xb(t))$ for all $t$.
	\end{theorem}

	\section{Sketch Proofs for RL with WGF}
	
	\begin{proof}[Proof of Proposition~\ref{prop:IPWGF}]
		
		We provide two methods for the proof.
		
		The first method directly uses property of gradient flows. Note that the WGF is defined as
		\begin{align*}
		\partial_{\tau} \mu_{\tau} &= -\nabla \cdot (\vb_{\tau} \mu_{\tau}) = \nabla \cdot \left(\mu_{\tau} \nabla(\frac{\delta F}{\delta \mu_{\tau}}(\mu_{\tau}))\right)\\
		& \triangleq -\nabla_{W}F(\mu_{\tau})~.
		\end{align*}
		Denote the inner product in the probability space induced by $W_2$ as $\langle\cdot,\cdot\rangle_{W}$, we have
		\begin{align}\label{eq:w_inner}
		\frac{\mathrm{d}}{\mathrm{d}\tau}F(\mu_{\tau}) &= \langle\nabla_{W}F(\mu_{\tau}), \frac{\mathrm{d}}{\mathrm{d}\tau}\mu_{\tau}\rangle_{W} \nonumber\\
		&= -\langle\nabla_{W}F(\mu_{\tau}), \nabla_{W}F(\mu_{\tau})\rangle_{W}~.
		\end{align}
		
		For any $\tau_1 \geq \tau_0$, integrating \eqref{eq:w_inner} over $[t_0, t_1]$, we have
		\begin{align*}
		&F(\mu_{\tau_1}) - F(\mu_{\tau_0}) \\
		=& -\int_{\tau_0}^{\tau_1}\langle\nabla_{W}F(\mu_{\tau}), \nabla_{W}F(\mu_{\tau})\rangle_{W} \mathrm{d}\tau \leq 0~,
		\end{align*}
		where the last inequality holds due to the positiveness of the norm operator. Consequently, we have $F(\mu_{\tau_1}) \leq F(\mu_{\tau_0})$, which means the energy functional $F$ decreases over time. In our case, the energy functional is defined as the KL divergence, which is convex in terms of distributions. As a result, evolving $\mu_{\tau}$ along the gradient flow would reach the global minimum of the energy functional, {\it i.e.}, $\lim_{\tau\rightarrow \infty}\mu_{\tau} = e^{J(\pi_{\thetab})}$.
		
		The second way uses property of the Fokker-Planck equation for diffusions. Since the WGF with energy functional in \eqref{eq:imp_energy} is equivalent to a Fokker-Planck equation. Specifically, according to Section~\ref{sec:wgf}, the solution of the gradient flow is described by the following Fokker-Planck equation:
		\begin{align}\label{eq:fp1}
		\partial_{\tau} \mu_{\tau} = \nabla\cdot \left(-\mu_{\tau}\nabla J(\pi_{\thetab}) + \nabla\cdot\left(\mu_{\tau}\right)\right)~,
		\end{align}
		On the other hand, it is well known that the unique invariant probability measure for the FP equation \eqref{eq:fp1} is:
		\begin{align*}
		\mu = e^{J(\pi_{\thetab})} = \lim_{\tau\rightarrow \infty}\mu_{\tau}~.
		\end{align*}
		This completes the proof.
	\end{proof}
	
	\begin{proof}[Proof of Proposition~\ref{prop:dpwgf}]
		The first part of Proposition~\ref{prop:dpwgf}, stating that $\pi(\ab|\sbb)$ converges to $p_{s, \pi}(\ab) \propto e^{Q(\ab, \sbb)}$, follows by the same argument as the proof of Proposition~\ref{prop:IPWGF}. Now we derive the soft Bellman equation.
		
		This follows from the definition of $Q(\ab_t, \sbb_t)$, {\it i.e.},
		\begin{align*}
		&Q(\ab_t, \sbb_t) \\
		= &r(\ab_t, \sbb_t) + \mathbb{E}_{(\sbb_{t+1}, \ab_{t+1}, \cdots) \sim (\rho_{\pi}, \pi)}\sum_{l=1}	^{\infty}\gamma^lr(\sbb_{t+l}, \ab_{t+l}) \\
		= &r(\ab_t, \sbb_t) + \gamma \mathbb{E}_{\sbb_{t+1} \sim \rho_{\pi}}\mathbb{E}_{\ab_{t+1} \sim \pi}\left[r(\ab_{t+1}, \sbb_{t+1}) \right.\\
		&\left.+ \mathbb{E}_{(\sbb_{t+2}, \ab_{t+2} \cdots) \sim (\rho_{\pi}, \pi)}\sum_{l=1}	^{\infty}\gamma^{l}r(\sbb_{t+1+l}, \ab_{t+1+l})\right]
		\end{align*}
		Since $\pi(\ab|\sbb) = e^{Q(\ab, \sbb) - V_{\pi}(\sbb)}$ where $V_{\pi}(\sbb)  =\int_{\mathcal{A}}Q(\ab, \sbb_{t+1})\mathrm{d}\ab$, we have
		\begin{align*}
		Q&(\ab_t, \sbb_t) = r(\ab_t, \sbb_t) \\
		&+ \gamma \mathbb{E}_{\sbb_{t+1} \sim \rho_{\pi}}\left[V_{\pi}(\sbb_{t+1}) - \mathcal{H}(\pi(\cdot|\sbb_{t+1}))\right]
		\end{align*}
	\end{proof}
	
	\section{More Details on Related Works}\label{supp:related}
	For reference, in addition to Section~\ref{sec:related}, we provide more details on the connection of our framework compared to existing methods.
	\paragraph{Connections with trust region methods}
	Trust Region methods can stabilize policy optimization in RL~\cite{NachumNXS:arxiv17, kakade2002natural}. 
	We can also show the connection between the proposed method and TRPO~\cite{schulman2015trust}. In TRPO, an objective function is maximized subjected to a constraint on the size of policy update. Specifically,
	\begin{align}\label{eq:trpo1}
	\max_{\phib}~~~&\hat{\mathbb{E}}_t\left[\dfrac{\pi^{\phi}(\cdot|\sbb)}{\pi^{\phi_{k-1}}(\cdot|\sbb) }\hat{A}_k\right]\\
	{\rm subject~to}~~~& \hat{\mathbb{E}}_t\left[\mbox{KL}[\pi^{\phi}(\cdot|\sbb), \pi^{\phi_{k-1}}(\cdot|\sbb)]\right]\leq \delta
	\end{align}
	Here, $\pi_{\phi}$ is a stochastic policy; $\phi_{k-1}$ is the vector of policy parameters before the k-th update; $\hat{A}_k$ is an estimator of the advantage function at timestep $k$. The theory of TRPO suggests using a penalty instead of a constraint, {\it i.e.}, solving the unconstrained optimization problem,
	\begin{align}\label{eq:trpo2}
	\max_{\phib}~~~&\hat{\mathbb{E}}_t\left[\dfrac{\pi^{\phi}(\cdot|\sbb)}{\pi^{\phi_{t-1}}(\cdot|\sbb) }\hat{A}_t-\beta \mbox{KL}[\pi^{\phi}(\cdot|\sbb), \pi^{\phi_{t-1}}(\cdot|\sbb)] \right]
	\end{align}
	
	In our proposed framework, the Wasserstein distance between $\pi^{\phi}(\cdot|\sbb)$ and $ \pi^{\phi_{k-1}}(\cdot|\sbb)$, a weaker metric than the KL divergence, constrains the update of a policy on a manifold endowed with the Wasserstein metric, and potentially leads to more robust solutions. This is evidenced by the development of Wasserstein GAN \cite{ArjovskyCB:arxiv17}. As a result, our framework can be regarded as a trust region based counterpart for solving the soft Q-learning problem~\cite{HaarnojaTAL:ICML17}.
	
	\paragraph{Connections with noisy exploration}
	In our framework, adding noise to the parameters, leading to noisy exploration can be interpreted as a special case of indirect policy learning with single particle. As shown in \cite{fortunato2017noisy, plappert2017parameter}, independent Gaussian noisy linear layer is defined as.
	\begin{align}
	y&\eqdef&(\mu^w+\sigma^w\odot\eps^w) x+ \mu^b  +\sigma^b\odot\eps^b ,
	\end{align}
	The parameters $\mu^w\in \mathbb{R}^{q\times p}$, $\mu^b\in \mathbb{R}^{q}$, $\sigma^w\in \mathbb{R}^{q\times p}$ and $\sigma^b\in \mathbb{R}^{q}$ are learnable whereas $\eps^w\in \mathbb{R}^{q\times p}$  and $\eps^b\in \mathbb{R}^{q}$ are  random noises, where $p$ and $q$ are the number of hidden units of connected layers.
	
	It corresponds to the maximum a posterior (MAP) with a Gaussian assumption on the posterior distributions of parameters (weight uncertainty). In our framework, we can explicitly \cite{LiuW:NIPS16} or implicitly \cite{blundell2015weight} define the weight uncertainty. Especially, when employing SGLD~\cite{WellingT:ICML11} to approximate the posterior distributions of parameters, we will using noisy gradient instead of noisy weights in the parameter space. 
	
	Previous work, such as DDPG~\cite{LillicrapHPHETSW:ICLR16}, adding noise to the action to encourage exploration can be regarded as a special case of DP-WGF. Adding noise in parameter space has shown superiority with action space, but the computational cost of employing particles to approximate parameter distribution is much higher than that of directly approximate policy distribution. Previous work \cite{fortunato2017noisy, plappert2017parameter}  made a trade-off and optimize the MAP instead of the distribution.   
	
	\section{Extensive Experiments}
	
	To optimize over the discretized WGF via the JKO scheme \eqref{eq:ito_discrete}, to need to specify the discretized stepsize $h$. In addition, we have an additional hyparameter $\lambda$ in the gradient formula of $W_2^2$. Also note that we can only evaluate the gradient of $W_2^2$ up to a constant, there needs to a parameter balancing the gradient of the energy functional $F$ and the Wasserstein term. We denote this hyparameter as $\epsilon$.
	In the experiments, if not explicitly stated, the default setting for these parameters are $\epsilon = 0.4$, and $\lambda = \mathtt{med}^2/\log M$. Here $\mathtt{med}$ is the median of the pairwise distance between particles of consecutive policies. 
	Adam~\cite{kingma2014adam} optimizer is used for all experiments, except the BNN regression, for which we use RMSPorp~\cite{hinton2012rmsprop}.
	\subsection{Comparative Evaluation}
	DP-WGF-V learns substantially faster than popular baselines on four tasks. In the Humanoid task, even though TRPO-GAE does not outperforms DP-WGF-V within the range depicted in the Figure \ref{fig:rlSVGD1}, it achieves good final rewards after more episodes. The quantitative results in our experiments are also comparable to results reported by other methods in prior work \cite{duan2016benchmarking, gu2016q, henderson2017deep, o2016pgq, mnih2016asynchronous}, showing sample efficiency and good performance. 
	\subsection{BNN for regression}\label{supp:exp1}
	For SVGD-based methods, we use a RBF kernel $\kappa(\thetav,\thetav') = \exp(-\|\thetav-\thetav'\|_2^2/m)$, with the bandwidth set to $m=\mathtt{med}^2/\log M$. Here $\mathtt{med}$ is the median of the pairwise distance between particles. 
	We use a single-layer BNN for regression tasks. Following \cite{blundell2015weight},   10 UCI public datasets are considered: 100 hidden units for 2 large datasets (Protein and YearPredict), and 50 hidden units for the other 8 small datasets. 
	%
	We repeat the experiments 20 times for all datasets except for Protein and YearPredict, which we repeat 5 times and once, respectively, for computation consideration~\cite{blundell2015weight}. 
	The experiment settings are almost identical to those in~\cite{blundell2015weight}, except that the prior of covariances follow $\InvGamma(1, 0.1)$. The batch size for the two large datasets is set to 1000, while it is 100 for the small datasets. The datasets are randomly split into 90\% training and 10\% testing. Table~\ref{tab:reg1} shows the complete results for different models on all the datasets.
	
	\begin{table*}[tb]
		\centering
		\begin{adjustbox}{scale=0.80,tabular=ccccccccc,center}
			& \multicolumn{4}{c}{Test RMSE $\downarrow$} &\multicolumn{4}{c}{Test Log likelihood $\uparrow$} \\
			Dataset & Dropout & PBP & SVGD& WGF & Dropout &  PBP &SVGD & WGF \\
			\hline
			Boston & 4.32 $\pm$ 0.29 & 3.01 $\pm$ 0.18 & 2.96$\pm$0.10&$\mathbf{ 2.46 \pm 0.34 }$ & -2.90 $\pm$ 0.07 & -2.57 $\pm$ 0.09 &-2.50$\pm$0.03&$\mathbf{ -2.40 \pm 0.10 }$\\
			Concrete&7.19 $\pm$ 0.12 &5.67 $\pm$ 0.09&5.32$\pm$0.10&$\mathbf{ 4.59 \pm 0.29 }$ &-3.39 $\pm$ 0.02 &-3.16 $\pm$ 0.02&-3.08$\pm$0.02&$\mathbf{ -2.95 \pm 0.06 }$\\
			Energy&2.65 $\pm$ 0.08&1.80 $\pm$ 0.05&1.37$\pm$0.05&$\mathbf{ 0.48 \pm 0.04 }$  &-2.39 $\pm$ 0.03&-2.04 $\pm$ 0.02&-1.77$\pm$0.02&$\mathbf{ -0.73 \pm 0.08 }$\\
			Kin8nm&0.10 $\pm$ 0.00&0.10 $\pm$ 0.00& $\mathbf{ 0.09 \pm 0.00 }$&$\mathbf{ 0.09 \pm 0.00 }$&0.90 $\pm$ 0.01&0.90 $\pm$ 0.01&$\mathbf{ 0.98\pm0.01}$& 0.97 $\pm$ 0.02 \\
			Naval&0.01 $\pm$ 0.00 &.01$\pm$ 0.00&0.00$\pm0.00$&$\mathbf{ 0.00 \pm 0.00 }$ &3.73 $\pm$ 0.12 &3.73 $\pm$ 0.01&4.09$\pm$0.01&$\mathbf{ 4.11 \pm 0.02 }$\\
			CCPP&4.33 $\pm$ 0.04&4.12$\pm$ 0.03&4.03$\pm$0.03&$\mathbf{ 3.88 \pm 0.06 }$ &-2.89 $\pm$ 0.02 &-2.80 $\pm$ 0.05&-2.82$\pm0.01$&$\mathbf{ -2.78 \pm 0.01 }$\\
			Winequality &0.65 $\pm$ 0.01 &0.64 $\pm$ 0.02& 0.61$\pm$0.01&$\mathbf{0.57 \pm 0.03}$ &-0.98 $\pm$ 0.01 &-0.97 $\pm$ 0.01& -0.93$\pm$0.01&$\mathbf{-0.87 \pm 0.04}$ \\
			Yacht&6.89 $\pm$ 0.67 & 1.02 $\pm$ 0.05&0.86$\pm$0.05&$\mathbf{ 0.56 \pm 0.16 }$ &-3.43 $\pm$ 0.16 &-1.63 $\pm$ 0.02&-1.23$\pm$0.04&$\mathbf{ -0.99 \pm 0.15 }$\\
			Protein&  4.84 $\pm$ 0.03 & 4.73 $\pm$ 0.01&4.61$\pm$0.01 & $\mathbf{4.24 \pm 0.02}$ & -2.99 $\pm$ 0.01 & -2.97 $\pm$ 0.00& -2.95$\pm$0.00& $\mathbf{-2.88 \pm 0.01}$\\
			YearPredict&9.03 $\pm$ NA & 8.88 $\pm$ NA& 8.68$\pm$ NA&$\mathbf{8.66 \pm NA}$ &-3.62 $\pm$ NA &-3.60$\pm$NA&-3.58 $\pm$ NA & $\mathbf{ -3.57 \pm NA }$\\
			\hline
		\end{adjustbox}
		\caption{Averaged predictions with standard deviations in terms of RMSE and log-likelihood on test sets.}
		\label{tab:reg1}
	\end{table*}
	
	\subsection{Toy example in a multi-goal environment}
	We use the similar toy example as in soft $Q$-learning to show that our proposed , where the environment is defined as a multi-modal distribution, 
	
	Figure \ref{fig:multi-goal} illustrates a 2D multi-goal environment. The left one shows trajectories from a policy learned with DP-WGF-V. The x and y axes correspond to 2D positions (states). The agent is initialized near the origin, and the first step of trajectory is omitted. Red dots are depicted goals and the environment is terminated once the distance between the agents and some goal meets predefined threshold. The level curves show the distance to the goal. 
	
	Q-values at three selected states (-2.5, 0), (0, 0), (2.5, 2.5) are presented on the right, depicted by level curves (yellow: high values, red: low values). The x and y axes correspond to 2D velocity (actions) bounded between -1 and 1. Actions sampled from the policy are shown as blue stars. The experiments shows that our methods have the ability to learn multi-goal policies while achieving better stability and rewards than soft-Q learning.
	
	\begin{figure}[t] 
		\centering
		\includegraphics[width=8cm]{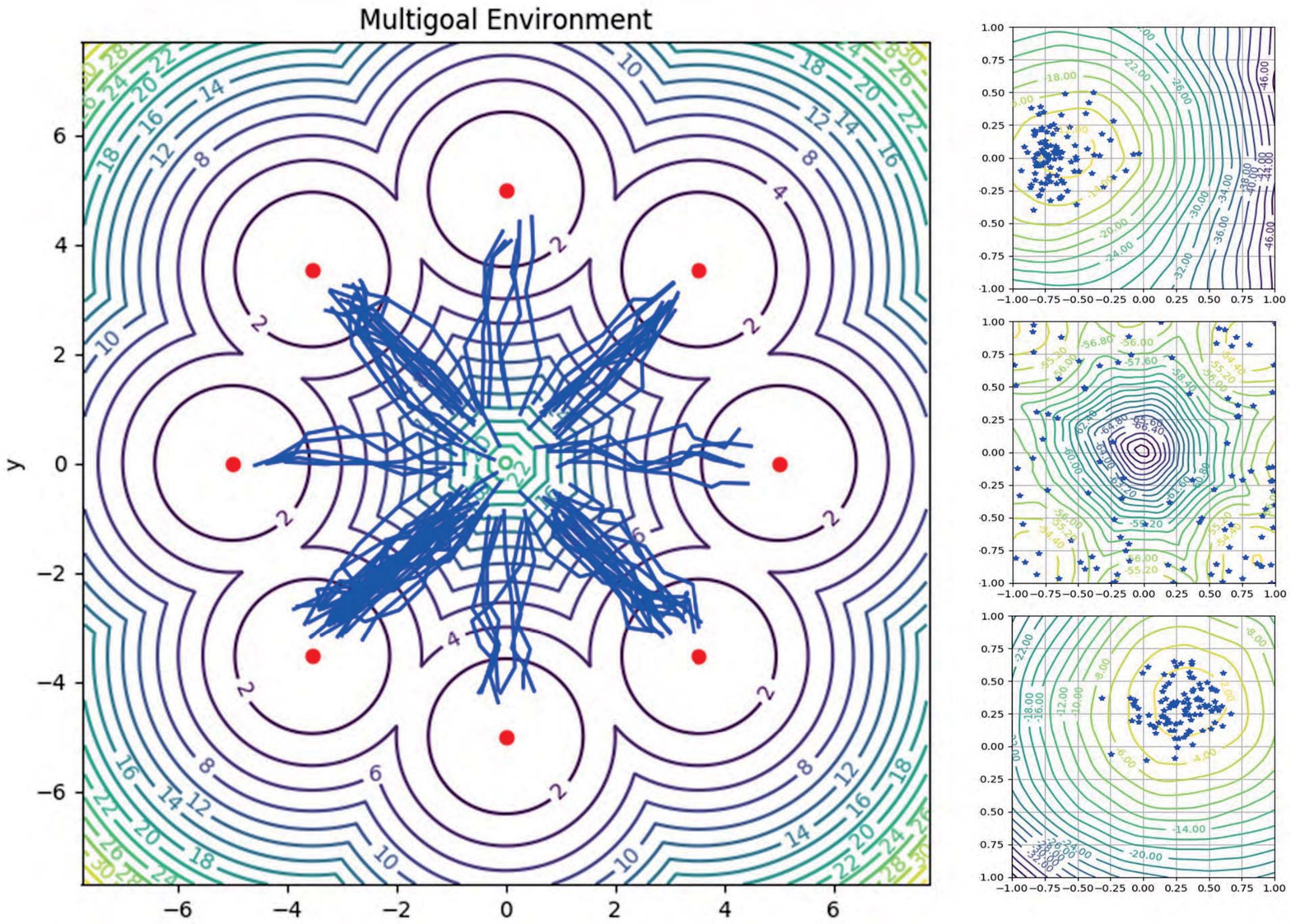}  
		\vspace{-3mm}
		\caption{DP-WGF-V on multi-goal Environment.}
		\label{fig:multi-goal}
	\end{figure}
	
	\subsection{Hyperparameter Sensitivity}
	\begin{figure}[h] \centering
		\begin{tabular}{c}
			\hspace{-7mm}
			\includegraphics[width=6cm]{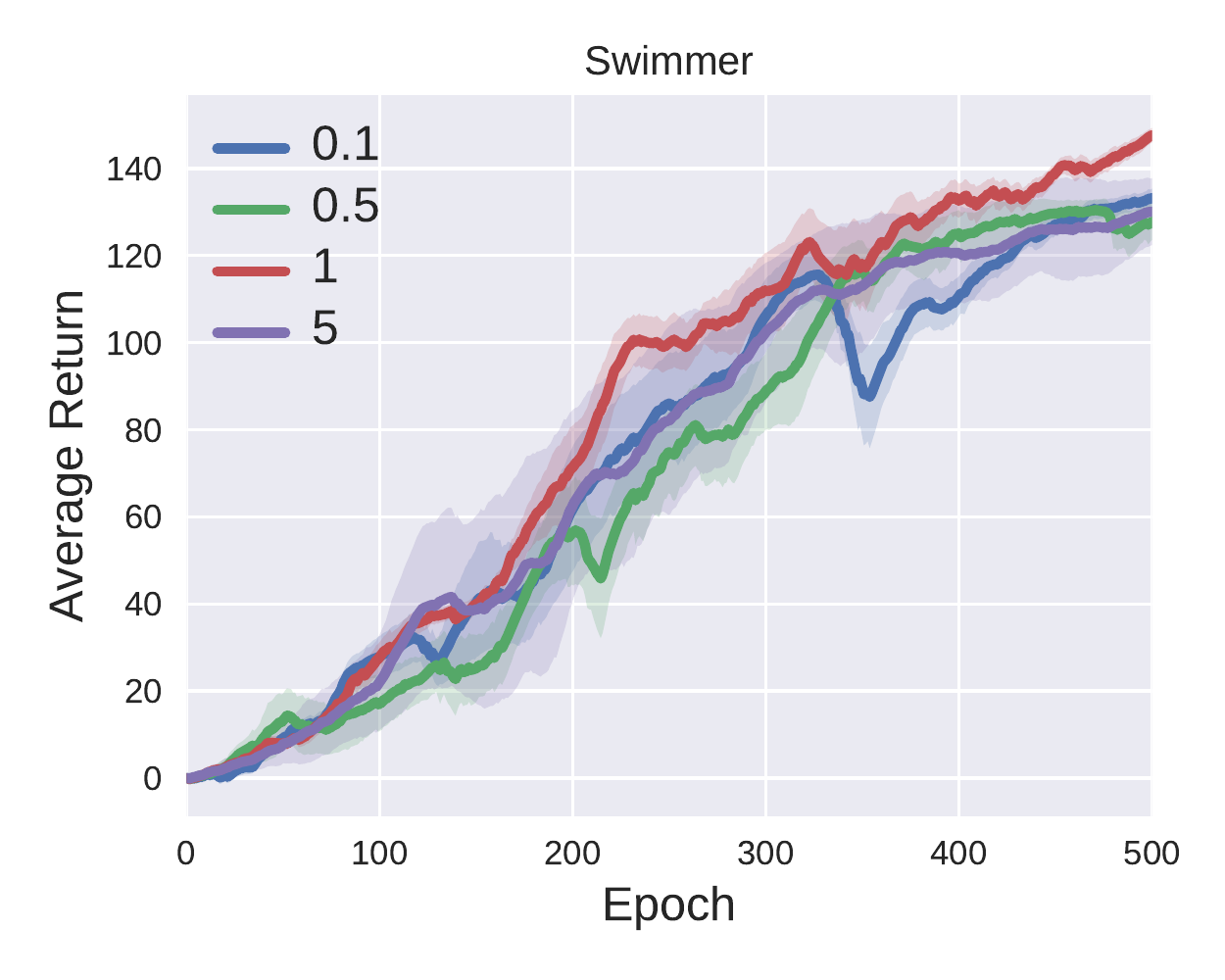}  
		\end{tabular} \vspace{-4mm}
		\caption{{\small Sensitivity of Hyper-parameters}}
		\vspace{-2mm}
		\label{fig:rlSVGD2_w2}
	\end{figure}	
	We further conduct experiments on Swimmer-v1 task to analysis the influence of different Wasserstein-2 scale $\epsilon$.  We run the algorithm for 500 epochs and 5 re-runs. Figure \ref{fig:rlSVGD2_w2}(a) shows the mean of average return against epoch. From the experiments, with appropriate $\epsilon$, the learning curves become more stable and achieves higher final rewards; while too large scale of $\epsilon$ will reduce the final rewards of policy, since the update size is excessively restricted. The results also show that the scale $\epsilon$  of Wasserstein trust-region is not parameter sensitive. 
	
	%

	\section{Implementation Details}
	\subsection{Smoothing previous policy}
	Towards Wasserstein-2 distance, we need to use consecutive to compute policies $W_2^2(\pi^{\overline{\phib}}(\cdot |\sbb_t),\pi^{\phib}(\cdot |\sbb_t))$. For the previous policy $\pi^{\overline{\phib}}(\cdot|\sbb_t)$, there are two strategy to get it. {\RN{1})} policy of last iteration, {\it i.e.} $\overline{\phib}=\phib_{k-1}$. {\RN{2})} moving average of prior policy, {\it i.e.} $\overline{\phib}=(1-\tau)\overline{\phib} + \tau\phib_{k-1}$. Empirically, when the learning curve is stable, (e.g. Half-Cheetah-v1), adopting strategy {\RN{1})} is helpful, and strategy {\RN{2})} will reduce the speed of convergence and may lead lower final rewards; Otherwise, strategy {\RN{2})} will help stabilize the training, and speed up the convergence.
	
	\begin{table}[H]
		\renewcommand{\arraystretch}{1.1}
		\centering
		\caption{Shared parameters of direct policy learning}
		\label{tab:shared_params}
		\vspace{1mm}
		\begin{tabular}{ r l| l }
			\hline
			Parameter & Symbol & Value\\
			\hline
			horizon  &&500\\
			batch size  	&  &   5000    \\
			learning rate &  &$5\!\times\!10^{-3}$ \\
			discount & $\gamma$& 0.99\\
			hidden units && [25, 16] \\					
			variance (prior) && 0.01\\
			temperature &$\alpha$& $\{6, 7, 8, 9, 10, 11\}$\\
			\hline
		\end{tabular}
	\end{table}
	
	\subsection{Indirect Policy learning}
	For the easy task, Cartpole, all agents are trained for 100 episodes. For the two complex tasks, Cartpole Swing-Up and Double Pendulum, all agents are trained up to 1000 episodes. SVPG and IP-WGF shared the same hyperparameters, except the temperature, for which we performed a grid search over $\alpha \in \{6, 7, 8, 9, 10, 11\}$.

	\setcounter{footnote}{0}
	\subsection{Direct-Policy learning}\label{supp:DPL}
	We use OpenAI gym\footnote{https://github.com/openai/baselines}~\cite{openaigym} and rllab\footnote{https://github.com/rll/rllab/tree/master/examples}~\cite{duan2016benchmarking} baselines implementations for TRPO and DDPG. SAC\footnote{https://github.com/haarnoja/sac} and Soft-Q\footnote{https://github.com/haarnoja/softqlearning} implementation are used, and we use recommended parameters.
	
	\paragraph{Hyperparameters}
	\label{app:hypers}
	\autoref{tab:shared_params} lists the common DP-WGF-V, DP-WGF, SAC and Soft-Q parameters used in the comparative evaluation in Figure~\ref{fig:rlSVGD1}, and \autoref{tab:wgf_results} lists the parameters that varied across the environments. For SAC, we use 4 components of mixture Gaussian. For DP-WGF and Soft-Q, 32 particles are used to approximate policy distributions.
	
	\begin{table}[H]
		\renewcommand{\arraystretch}{1.1}
		\centering
		\caption{Shared parameters of indirect policy learning}
		\label{tab:shared_params}
		\vspace{1mm}
		\begin{tabular}{ r l| l }
			\hline
			Parameter & Symbol & Value\\
			\hline
			horizon&&1000\\
			batch size && 64\\
			learning rate & & $3 \cdot 10^{-4}$\\
			discount & $\gamma$ &  0.99\\
			target smoothing coefficient &$\tau$& 0.01\\
			number of layers (3 networks) & & 2\\
			number of hidden units per layer && 128\\
			gradient steps && 1\\
			scale of Wasserstein trust-region && 0.4\\
			\hline
		\end{tabular}
	\end{table}
	\vspace{-0.4in}
	\begin{table}[H]
		\renewcommand{\arraystretch}{1.1}
		\centering
		\caption{Environment Specific Parameters}
		\label{tab:env_params}
		\vspace{1mm}
		\begin{tabular}{ r|l l l  }
			\hline
			Environment 		&DoFs	&Reward Scale  & Replay Pool \\ 
			\hline
			Swimmer 	&2 		& 100          			   & $10^6$\\
			Hopper-v1 			&3		& 1            			   & $10^6$\\
			Walker2d-v1 		&6 		& 3           			   & $10^6$\\
			Humanoid 	&21 	& 3         			   & $10^6$\\
			\hline
		\end{tabular}
	\end{table}
	
	\section{Demos}
	Demos of our framework on a set of RL tasks can be accessed online via:
	\href{https://sites.google.com/view/wgf4rl/}{https://sites.google.com/view/wgf4rl/} .

	\section{Algorithm Details}\label{supp:alg}
	For completeness, we list the detailed algorithms for IP-WGF, DP-WGF and DP-WGF-V in Algorithms~\ref{alg:dpwgf}, \ref{alg:ipwgf} and \ref{alg:dpwgfv}, respectively.
	
	\begin{algorithm}[H]
		\caption{DP-WGF}
		\label{alg:dpwgf}
		\begin{algorithmic}
			\REQUIRE {$\mathcal{D} = \emptyset$; initialize $ \thetab, \phib \sim $ some (prior) distribution. Target parameters: $\overline{\thetab} \leftarrow \thetab$, $\overline{\phib} \leftarrow \phib$}
			\FOR{each epoch}
			\FOR{each t}
			\vspace{1mm}
			\STATE  $\mathtt{\color{blue} \%~Collect~~expereince} $
			\STATE Sample an action $\ab_t$~from policy $\pi^{\phib}(\cdot |\sbb_t)$.
			\STATE Sample next state from the environment: $\sbb_{t+1} \sim p_{\sbb} (\sbb_{t+1}|\sbb_{t}, \ab_{t})$ 
			\STATE Save the new experience in the replay memory: $\mathcal{D} \leftarrow  \mathcal{D}\cup \{\sbb_{t}, \ab_t, r(\sbb_{t}, \ab_t), \sbb_{t+1}\}$ 
			\STATE $\mathtt{\color{blue} }$
			\%~Sample~~from~~the~~replay~~memory 
			\STATE $\{(\sbb_{t}^{(i)}, \ab_t^{(i)}, r_t^{(i)}, \sbb_{t+1}^{(i)})\}^{N}_{i=0} \sim \mathcal{D}$.
			\STATE $\mathtt{\color{blue}  \%~Update~~Q~~function} $
			\STATE Compute empirical values $\hat{V}^{\overline{\thetab}}(\sbb_{t+1}^{(i)})$ 
			\STATE Compute empirical gradient $\hat{\nabla}_{\thetab} J_Q(\thetab)$
			\STATE Update $\thetab$ according to it using ADAM
			\STATE $\mathtt{\color{blue}  \%~Update~~policy} $
			\STATE Compute $W_2^2(\pi^{\overline{\phib}}(\cdot |\sbb_t),\pi^{\phib}(\cdot |\sbb_t))$,
			\STATE Compute empirical gradient $\hat{\nabla}_{\phib} J_\pi(\phib)$
			\STATE Update prior policy parameters: $\overline{\phib} \leftarrow \phib$
			\STATE Update $\thetab$ according to it using ADAM
			\STATE $\mathtt{\color{blue}  \%~Update~~target} $
			\STATE Update~~target~~Q~function~~parameters: 
			\STATE ~~~~~~~~~~$\overline{\thetab} \leftarrow \tau\thetab + (1-\tau)\overline{\thetab}$
			\ENDFOR
			\ENDFOR
		\end{algorithmic}
	\end{algorithm}
	\vspace{-4mm}
	\begin{algorithm}[H]
		\caption{IP-WGF}
		\label{alg:ipwgf}
		\begin{algorithmic}
			\REQUIRE {Initialize policy particles $ \Thetab \sim $ some (prior) distribution as a Bayesian neural network.}
			\FOR{each iteration}
			\STATE Reset FIFO replay pool $R$
			\FOR{each timestep t in episodes}
			\STATE{Sample $\ab_t$ from $\pi_\phi(\cdot|\sbb_t)$}
			\STATE Sample next state from the environment: $\sbb_{t+1} \sim p_{\sbb} (\sbb_{t+1}|\sbb_{t}, \ab_{t})$
			\STATE Save experience in to FIFO replay pool $R$:  
			\FOR{each particles $\thetab^{(i)}\in\Thetab$}
			\STATE Compute $W_2^2(\overline{\Thetab}^{(i)}, \Thetab^{(i)})$
			\STATE Compute empirical gradient $\nabla_{\thetab^{(i)}}J(\pi_{\thetab^{(i)}})$
			\STATE Save current particles $\overline{\thetab}^{(i)}\leftarrow\thetab^{(i)}$
			\STATE Update policy particle $\thetab^{(i)}$
			\ENDFOR
			\ENDFOR
			\ENDFOR
		\end{algorithmic}
	\end{algorithm}

	\begin{algorithm}[H]
		\caption{DP-WGF-V}
		\label{alg:dpwgfv}
		\begin{algorithmic}
			\REQUIRE {$\mathcal{D} = \emptyset$; initialize $ \thetab, \phib, \psib \sim $ some (prior) distribution. Target parameters: $\overline{\thetab} \leftarrow \thetab$, $\overline{\phib} \leftarrow \phib$}
			\FOR{each epoch}
			\FOR{each t}
			\vspace{1mm}
			\STATE  $\mathtt{\color{blue} \%~Collect~~expereince} $
			\STATE Sample an action $\ab_t$~from policy $\pi^{\phib}(\cdot |\sbb_t)$.
			\STATE Sample next state from the environment: $\sbb_{t+1} \sim p_{\sbb} (\sbb_{t+1}|\sbb_{t}, \ab_{t})$ 
			\STATE Save the new experience in the replay memory: $\mathcal{D} \leftarrow  \mathcal{D}\cup \{\sbb_{t}, \ab_t, r(\sbb_{t}, \ab_t), \sbb_{t+1}\}$ 
			\STATE $\mathtt{\color{blue} 
				\%~Sample~~from~~the~~replay~~memory} $
			\STATE $\{(\sbb_{t}^{(i)}, \ab_t^{(i)}, r_t^{(i)}, \sbb_{t+1}^{(i)})\}^{N}_{i=0} \sim \mathcal{D}$.
			\STATE $\mathtt{\color{blue}  \%~Update~~Q~function} $ 
			\STATE Compute empirical gradient $\hat{\nabla}_{\thetab} J_Q(\thetab)$
			\STATE Update $\thetab$ according to it using ADAM
			\STATE $\mathtt{\color{blue}  \%~Update~~value~function} $ 
			\STATE Compute empirical gradient $\hat{\nabla}_{\psib} J_V(\psib)$
			\STATE Update $\psib$ according to it using ADAM
			\STATE $\mathtt{\color{blue}  \%~Update~~policy } $
			\STATE Compute $W_2^2(\pi^{\overline{\phib}}(\cdot |\sbb_t),\pi^{\phib}(\cdot |\sbb_t))$,
			\STATE Compute empirical gradient $\hat{\nabla}_{\phib} J_\pi^\phi$
			\STATE Update prior policy parameters: $\overline{\phib} \leftarrow \phib$
			\STATE Update $\phib$ according to it using ADAM
			\STATE $\mathtt{\color{blue}  \%~Update~~target} $
			\STATE Update target value parameters: 
			\STATE ~~~~~~~~~~$\overline{\psib} \leftarrow \tau\psib + (1-\tau)\overline{\psib}$
			
			\ENDFOR
			\ENDFOR
		\end{algorithmic}
	\end{algorithm}
\end{document}